\documentclass[11pt]{article}

\usepackage{fullpage}
\usepackage{natbib}
\usepackage{algorithm}
\usepackage[noend]{algorithmic}
\usepackage{amsmath,amsthm,amsfonts}
\usepackage{amsmath}
\usepackage{hyperref}
\usepackage{color}
\usepackage{mathrsfs}
\usepackage{enumitem}
\usepackage{bm}
\usepackage{multirow}
\usepackage{booktabs}
\usepackage{makecell}
\usepackage{graphicx}
\usepackage{subfigure}
\usepackage{caption}

\newcommand{\defeq}{\mathrel{\mathop:}=}

\newcommand{\vect}[1]{\ensuremath{\mathbf{#1}}}
\newcommand{\mat}[1]{\ensuremath{\mathbf{#1}}}

\newcommand{\grad}{\nabla}
\newcommand{\hess}{\nabla^2}
\newcommand{\argmin}{\mathop{\rm argmin}}

\newcommand{\norm}[1]{\|{#1}\|}
\newcommand{\fnorm}[1]{\|{#1}\|_{\text{F}}}

\newcommand{\tr}{\text{tr}}

\newcommand{\rank}{\text{rank}}

\newcommand{\trans}{^{\top}}
\newcommand{\poly}{\text{poly}}
\newcommand{\polylog}{\text{polylog}}

\newcommand{\proj}{\mathcal{P}}

\newcommand{\projX}{\mathcal{P}_{\mathcal{X}^\star}}

\newcommand{\R}{\mathbb{R}}

\newcommand{\A}{\mat{A}}
\newcommand{\B}{\mat{B}}
\newcommand{\C}{\mat{C}}
\newcommand{\T}{\mat{T}}

\newcommand{\I}{\mat{I}}
\newcommand{\M}{\mat{M}}
\newcommand{\D}{\mat{D}}
\newcommand{\U}{\mat{U}}
\newcommand{\V}{\mat{V}}

\newcommand{\Y}{\mat{Y}}

\newcommand{\e}{\vect{e}}
\renewcommand{\u}{\vect{u}}
\renewcommand{\v}{\vect{v}}
\newcommand{\w}{\vect{w}}
\newcommand{\x}{\vect{x}}
\newcommand{\y}{\vect{y}}

\renewcommand{\H}{\mathcal{H}}

\newcommand{\cn}{\kappa}
\newcommand{\nn}{\nonumber}

\newtheorem{theorem}{Theorem}
\newtheorem{lemma}[theorem]{Lemma}
\newtheorem{corollary}[theorem]{Corollary}

\theoremstyle{definition}
\newtheorem{definition}{Definition}
\newtheorem{assumption}{Assumption}
\renewcommand\theassumption{A\arabic{assumption}}

\newcommand{\ca}{\hat{c}}
\newcommand{\Ts}{T'}
\newcommand{\Tt}{T''}

\newcommand{\ugrad}{\mathscr{G}}
\newcommand{\ufun}{\mathscr{F}}
\newcommand{\uspace}{\mathscr{S}}
\newcommand{\utime}{\mathscr{T}}

\renewcommand{\S}{\mathcal{S}}
\newcommand{\Scomp}{\mathcal{S}^c}
\newcommand{\balpha}{\bm{\alpha}}
\newcommand{\bbeta}{\bm{\beta}}
\newcommand{\bdelta}{\bm{\delta}}
\newcommand{\logterms}{\frac{d\cn}{\delta}}
\newcommand{\overh}{\ca \log(\logterms)}

\newcommand{\sigstarl}{\sigma^\star_1}
\newcommand{\sigstarr}{\sigma^\star_r}
\newcommand{\mR}{\mat{R}}
\newcommand{\mZ}{\mat{Z}}
\newcommand{\la}{\langle}
\newcommand{\ra}{\rangle}
\newcommand{\cXstar}{\mathcal{X}^\star}
\newcommand{\cXe}{\mathcal{X}_{\text{escape}}}
\newcommand{\cXs}{\mathcal{X}_{\text{stuck}}}

\newcommand{\ball}{\mathbb{B}}
\newcommand{\EFSP}{$\epsilon$-first-order stationary point}
\newcommand{\ESSP}{$\epsilon$-second-order stationary point}

\begin{document}

% \title{\textbf{Gradient Descent Finds Local Minima as Fast as It Finds First-Order Stationary Point}}
\title{\textbf{How to Escape Saddle Points Efficiently}}

\author{Chi Jin\footnote{University of California, Berkeley. Email: chijin@cs.berkeley.edu} \and 
Rong Ge\footnote{Duke University. Email: rongge@cs.duke.edu} \and
Praneeth Netrapalli\footnote{Microsoft Research, India. Email: praneeth@microsoft.com} \and
Sham M. Kakade\footnote{University of Washington. Email: sham@cs.washington.edu} \and
Michael I. Jordan\footnote{University of California, Berkeley. Email: jordan@cs.berkeley.edu}
}

\maketitle

\begin{abstract}
This paper shows that a perturbed form of gradient descent converges to a second-order stationary point in a number iterations which depends only poly-logarithmically on dimension (i.e., it is almost ``dimension-free''). The convergence rate of this procedure matches the well-known convergence rate of gradient descent to first-order stationary points, up to log factors. When all saddle points are non-degenerate, all second-order stationary points are local minima, and our result thus shows that perturbed gradient descent can escape saddle points almost for free.  

Our results can be directly applied to many machine learning applications, including deep learning. As a particular concrete example of such an application, we show that our results can be used directly to establish sharp global convergence rates for matrix factorization. Our results rely on a novel characterization of the geometry around saddle points, which may be of independent interest to the non-convex optimization community.
\end{abstract}

%!TEX root = main.tex

\section{Introduction}

Given a function $f: \R^d \rightarrow \R$, gradient descent aims
to minimize the function via the following iteration:
\begin{equation*}
\x_{t+1} = \x_{t} - \eta \grad f(\x_t),
\end{equation*}
where $\eta > 0$ is a step size.  Gradient descent and its variants 
(e.g., stochastic gradient) are widely used in machine learning 
applications due to their favorable computational properties. 
This is notably true in the deep learning setting, where gradients
can be computed efficiently via back-propagation~\citep{rumelhart1988learning}.

Gradient descent is especially useful in high-dimensional settings 
because the number of iterations required to reach a point with 
small gradient is independent of the dimension (``dimension-free'').
% \textcolor{blue}{R: I wanted to remove the ``cheap computational cost per iteration'' because it is not really super cheap, it is really not feasible if you don't run the SGD version. I know you wanted to compare it to second order but I think the main point is the dimension dependency here.}  
More precisely, for a function that is $\ell$-gradient Lipschitz (see Definition \ref{def:smooth}), it is well known that gradient descent finds an $\epsilon$-first-order stationary point (i.e., a point $\x$ with $\norm{\grad f(\x)} \le \epsilon$) within $\ell (f(\x_0) - f^\star)/\epsilon^2$ iterations \citep{nesterov1998introductory}, where $\x_0$ is the initial point and $f^\star$ is 
the optimal value of $f$. This bound does not depend on the dimension of $\x$. In convex optimization, finding an $\epsilon$-first-order stationary point is equivalent to finding an approximate global optimum.% which does not depend on the dimension of $\x$.

In non-convex settings, however, convergence to first-order stationary points is not satisfactory. For non-convex functions, first-order stationary points can be global minima, local minima, saddle points or even local maxima. Finding a global minimum can be hard, but fortunately, for many non-convex problems, it is sufficient to find a local minimum. 
% \textcolor{blue}{I moved this in front.} 
Indeed, a line of recent results show that, in many problems of interest, either all local minima are global minima (e.g., in tensor decomposition \citep{ge2015escaping}, dictionary learning \citep{sun2016complete}, phase retrieval \citep{sun2016geometric}, matrix sensing \citep{bhojanapalli2016global,park2016non}, matrix completion \citep{ge2016matrix}, and certain classes of deep neural networks \citep{kawaguchi2016deep}).  Moreover, there are suggestions that in more
general deep newtorks most of the local minima are as good as global 
minima~\citep{choromanska2014loss}.

On the other hand, saddle points (and local maxima) can correspond to highly suboptimal solutions in many problems~\citep[see, e.g.,][]{jain2015computing,sun2016geometric}.
Furthermore, \citet{dauphin2014identifying} argue that saddle points are ubiquitous in high-dimensional, non-convex optimization problems, and are thus the main bottleneck in training neural networks.
Standard analysis of gradient descent cannot distinguish between saddle points and local minima, leaving open the possibility that gradient descent may get stuck at saddle points, either asymptotically or for a sufficiently long time so 
as to make training times for arriving at a local minimum infeasible.  ~\citet{ge2015escaping} showed that by adding noise at each step, gradient descent can escape all saddle points in a polynomial number of iterations, provided that the objective function satisfies the strict saddle property (see Assumption \ref{as:strict_saddle}). ~\citet{lee2016gradient} proved that under similar conditions, gradient descent with random initialization avoids saddle points even without adding noise. However, this result does not bound the number of steps needed to reach a local minimum.

% Nonconvex function, convergence to first order stationary points are not satisfactory: can be saddle points, and local minima.
% Deep learning: local minima are equally good.
% Matrix sensing, matrix factorization, etc: all local minima are global minima.

% ~

% Saddle points are non-desirable (matrix factorization). Saddle point problematic (deep learning paper).
% Previous gradient descent escape saddle point [Ge et al.] add perturbation poly rates of iterations.
% [Lee et al.] no rates. Still far away from being efficient, comparing to the dimension independent iteration complexity
% of gradient descent in convex optimization, these results still leave room for tremedous overhead casuing by saddle point.
% On the other hand, for several specific applications, much faster convergence rates of gradient descent are known [squareroot, PCA...]

% ~

Though these results establish that gradient descent can find local minima in a polynomial number of iterations, they are still far from being efficient. For instance, the number of iterations required in~\citet{ge2015escaping} is at least $\Omega(d^4)$, where $d$ is the underlying dimension.
This is significantly suboptimal compared to rates of convergence to first-order stationary points, where the iteration complexity is dimension-free.
%previous results still leave possibilities for tremedous computational overhead causing by saddle points.
This motivates the following question: \textbf{Can gradient descent escape saddle points and converge to local minima in a number of iterations that is (almost) dimension-free?}

In order to answer this question formally, this paper investigates the complexity of finding
% \jccomment{why not key problem? also convergence to second-order stationary point, but find \ESSP} 
$\epsilon$-second-order stationary points.
For $\rho$-Hessian Lipschitz functions (see Definition \ref{def:HessianLip}), these points are defined as~\citep{nesterov2006cubic}:
\begin{equation*}
\norm{\grad f(\x)} \le \epsilon, \qquad\text{and}\qquad \lambda_{\min}(\hess f(\x)) \ge - \sqrt{\rho \epsilon}.
\end{equation*}
Under the assumption that all saddle points are strict (i.e., for any saddle point $\x_s$, $\lambda_{\min}(\hess f(\x_s))<0$), all second-order stationary points ($\epsilon=0$) are local minima. Therefore, convergence to second-order stationary points is equivalent to convergence to local minima.% To our best knowledge, no sharp analysis is yet known to show how fast gradient descent converges to second-order stationary point. 

\begin{algorithm}[t]
\caption{Perturbed Gradient Descent (Meta-algorithm) }\label{algo:meta} 
% $(\x_0, \eta, r)$
\begin{algorithmic}
\FOR{$t = 0, 1, \ldots $}
\IF{perturbation condition holds} 
\STATE $\x_t \leftarrow \x_t + \xi_t, \qquad \xi_t \text{~uniformly~} \sim \ball_0(r)$
\ENDIF
\STATE $\x_{t+1} \leftarrow \x_t - \eta \grad f (\x_t)$
\ENDFOR
\end{algorithmic}
\end{algorithm}

This paper studies gradient descent with phasic perturbations (see Algorithm~\ref{algo:meta}). For $\ell$-smooth functions that are also Hessian Lipschitz, we show that perturbed gradient descent will converge to an $\epsilon$-second-order stationary point in $\tilde{O}(\ell (f(\x_0) - f^\star)/\epsilon^2)$, where $\tilde{O}(\cdot)$ hides polylog factors. This guarantee is almost dimension free (up to $\polylog(d)$ factors), answering the above highlighted question affirmatively. Note that this rate is exactly the same as the well-known convergence rate of gradient descent to first-order stationary points~\citep{nesterov1998introductory}, up to log factors. Furthermore, our analysis admits a maximal step size of up to $\Omega(1/\ell)$, which is the same as that in analyses for first-order stationary points.

As many real learning problems present strong \emph{local} geometric 
properties, similar to strong convexity in the global setting~\citep[see, e.g.][]{bhojanapalli2016global, sun2016guaranteed, zheng2016convergence}, 
it is important to note that our analysis naturally takes advantage 
of such local structure.  We show that when local strong
convexity is present, the $\epsilon$-dependence goes from a polynomial 
rate, $1/\epsilon^2$, to linear convergence, $\log (1/\epsilon)$. 
%Should I say competitive to previous local rate + smart initialization
As an example, we show that sharp global convergence rates can be
obtained for matrix factorization as a direct consequence of our
analysis.

\subsection{Our Contributions}
% In this work, we present the first sharp analysis showing under mild conditions, simple gradient descent can find local minima of a non-convex function as easy as it finds first-order stationary points. 
This paper presents the first sharp analysis that shows that (perturbed) 
gradient descent finds an approximate second-order stationary 
point in at most $polylog(d)$ iterations, thus escaping all 
saddle points efficiently. Our main technical contributions are as follows:
\begin{itemize}
\item For $\ell$-gradient Lipschitz, $\rho$-Hessian Lipschitz functions (possibly non-convex), gradient descent with appropriate perturbations finds an \ESSP~ in $\tilde{O}(\ell (f(\x_0) - f^\star)/\epsilon^2)$ iterations. This rate matches the well-known convergence rate of gradient descent to first-order stationary points up to log factors.
\item Under a strict-saddle condition (see Assumption \ref{as:strict_saddle}), this convergence result directly applies for finding local minima. This means that gradient descent can escape all saddle points with only logarithmic overhead in runtime.
\item When the function has local structure, such as local strong 
convexity (see Assumption \ref{as:sc}), the above results can be further improved to linear convergence. We give sharp rates that are comparable to previous problem-specific local analysis of gradient descent with smart initialization (see Section \ref{sec:related}).
\item All the above results rely on a new characterization of the
geometry around saddle points: points from where gradient descent gets stuck at a saddle point constitute a thin ``band.'' We develop novel techniques to bound the volume of this band. As a result, we can show that after a random perturbation the current point is very unlikely to be in the ``band''; hence, 
efficient escape from the saddle point is possible (see Section \ref{sec:sketch}). 
% novel Our technics is novel. instead of focusing on analysis of dynamic from any starting point,
%    we analyze two points nearby, saying if one can not escape from saddle point, the other one must escape.
%    we expect these technics very helpful in non-convex optimization for algorithms with perturbation in nature.
\end{itemize}

\subsection{Related Work} \label{sec:related}
Over the past few years, there have been many problem-specific convergence results for non-convex optimization. One line of work requires a smart initialization algorithm to provide a coarse estimate lying inside a local neighborhood, from which popular local search algorithms enjoy fast local convergence~\citep[see, e.g.,][]{netrapalli2013phase,candes2015phase,sun2016guaranteed,bhojanapalli2016global}.
% \jccomment{Praneeth, maybe you can help add some representative citations here.} \praneeth{I think this is good?}. 
While there are not many results  that show global convergence for non-convex problems,~\citet{jain2015computing} show that gradient descent yields global convergence rates for matrix square-root problems. Although these results give strong guarantees, the analyses are heavily tailored to specific problems, and it is unclear how to generalize them to a wider class of non-convex functions.

\begin{table}[t]
  \begin{center}
    {\renewcommand{\arraystretch}{1.3}
    \begin{tabular}{  >{\centering\arraybackslash}m{1.9in} >{\centering\arraybackslash}m{1.2in} >{\centering\arraybackslash}m{1.9in} }
    %\begin{tabular}{  K{3cm}  c  c  c }
       \toprule
\textbf{Algorithm} & \textbf{Iterations} & \textbf{Oracle} \\ 
      \midrule
\citet{ge2015escaping} & $O(\poly(d/\epsilon))$ & Gradient \\
\citet{levy2016power} & $O(d^3\cdot\poly(1/\epsilon))$ & Gradient \\
\textbf{This Work} & $O(\log^4(d)/\epsilon^2)$ & Gradient \\ 
      \midrule
\citet{agarwal2016finding} & $O(\log(d)/\epsilon^{7/4})$ & Hessian-vector product \\
\citet{carmon2016accelerated} & $O(\log(d)/\epsilon^{7/4})$ & Hessian-vector product \\
\citet{carmon2016gradient} & $O(\log(d)/\epsilon^2)$ & Hessian-vector product \\ 
\midrule
\citet{nesterov2006cubic} & $O(1/\epsilon^{1.5})$ & Hessian \\
\citet{curtis2014trust} & $O(1/\epsilon^{1.5})$ & Hessian\\
    \bottomrule
    \end{tabular}
      \caption{Oracle model and iteration complexity to second-order stationary point}
      \label{tab:main}
    }
  \end{center}
  % \praneeth{I think it will be cleaner to make the dependence on smoothness paramters explicit here.}
\end{table}

For general non-convex optimization, there are a few previous results on finding second-order stationary points. These results can be divided into the following three categories, where, for simplicity of presentation, we only highlight 
dependence on dimension $d$ and $\epsilon$, assuming that all other problem 
parameters are constant from the point of view of iteration complexity:

\noindent\textbf{Hessian-based:} Traditionally, only second-order optimization methods were known to converge to second-order stationary points. These algorithms rely on computing the Hessian to distinguish between first- and second-order stationary points. \citet{nesterov2006cubic} designed a cubic regularization algorithm which converges to an \ESSP~in $O(1/\epsilon^{1.5})$ iterations. Trust region algorithms \citep{curtis2014trust} can also achieve the same performance if the parameters are chosen carefully. These algorithms typically require 
the computation of the inverse of the full Hessian per iteration, which can be very expensive.

\noindent\textbf{Hessian-vector-product-based:} 
A number of recent papers have explored the possibility of using only 
Hessian-vector products instead of full Hessian information in order to find second-order stationary points. These algorithms require a Hessian-vector product oracle: given a function $f$, a point $\x$ and a direction $\u$, the oracle returns $\nabla^2 f(\x) \cdot \u$. \citet{agarwal2016finding} and \citet{carmon2016accelerated} presented accelerated algorithms that can find an \ESSP~in $O(\log d/\epsilon^{7/4})$ steps. Also, \citet{carmon2016gradient} showed by running gradient descent as a subroutine to solve the subproblem of cubic regularization (which requires Hessian-vector product oracle), it is possible to find an \ESSP in $O(\log d/\epsilon^2)$ iterations. 
%Computing the gradient of the cubic regularization subproblem requires the same Hessian-vector product oracle.
In many applications such an oracle can be implemented efficiently, in roughly the same complexity as the gradient oracle. Also, when the function has a Hessian Lipschitz property such an oracle can be approximated by differentiating the gradients at two very close points (although this may suffer from numerical issues, thus is seldom used in practice).

% However, computing the full Hessian can be very expensive. 
\noindent\textbf{Gradient-based:}
Another recent line of work shows that it is possible to converge to a second-order stationary point without any use of the Hessian. These methods feature simple computation per iteration (only involving gradient operations), and are closest to the algorithms used in practice. \citet{ge2015escaping} showed that stochastic gradient descent could converge to a second-order stationary point in $\poly(d/\epsilon)$ iterations, with polynomial of order at least four. This was improved in \citet{levy2016power} to $O(d^3\cdot \poly(1/\epsilon))$ using normalized gradient descent. The current paper improves on both results by showing that
%the bound for gradient descent algorithms: 
perturbed gradient descent can actually find an \ESSP~in $O(\polylog (d)/\epsilon^2)$ steps, which matches the guarantee for converging to first-order stationary points up to $\polylog$ factors. 
% For the cubic regularization subproblem, \citet{carmon2016gradient} showed gradient descent converges to an \ESSP in $1/\epsilon^2$ steps.

%!TEX root = main.tex

\section{Preliminaries}

In this section, we will first introduce our notation, and then present some definitions and existing results in optimization which will be used later.

\subsection{Notation}
We use bold upper-case letters $\A, \B$ to denote matrices and bold lower-case letters $\x, \y$ to denote vectors. $\A_{ij}$ means the $(i, j)^{\text{th}}$ entry of matrix $\A$. 
% $\norm{\x}$ denotes the $\ell_2$-norm of vector $\x$, and $\norm{\A}, \fnorm{\A}$ denotes the the spectral norm and the Frobenius norm of matrix $\A$.
For vectors we use $\norm{\cdot}$ to denote the $\ell_2$-norm, and for matrices we use $\norm{\cdot}$ and $\fnorm{\cdot}$ to denote spectral norm and Frobenius norm respectively. We use $\sigma_{\max}(\cdot), \sigma_{\min}(\cdot), \sigma_i(\cdot)$ to denote the largest, the smallest and the $i$-th largest singular values respectively, and $\lambda_{\max}(\cdot), \lambda_{\min}(\cdot), \lambda_i(\cdot)$ for corresponding eigenvalues.

For a function $f: \R^d \rightarrow \R$, we use $\grad f(\cdot)$ and  $\hess f(\cdot)$ to denote its gradient and Hessian, and $f^\star$ to denote the global minimum of $f(\cdot)$. 
% Other than Section \ref{sec:related}, 
We use notation $O(\cdot)$ to hide only absolute constants which do not depend on any problem parameter, and notation $\tilde{O}(\cdot)$ to hide only absolute constants and log factors. 
% \praneeth{I think it will be cleaner to make the dependence on smoothness parameters in Table~\ref{tab:main} and edit this statement} \jccomment{Then I also need to add function value dependence, maybe too complicated to compare}.\praneeth{The issue with this is that $O()$ is not just hiding constants but also problem dependent parameters. May be mention this explicitly in the caption to the table.} 
We let $\ball^{(d)}_\x(r)$ denote the d-dimensional ball centered at $\x$ with radius $r$; when it is clear from context, we simply denote it as $\ball_\x(r)$. We use $\proj_{\mathcal{X}}(\cdot)$ to denote projection onto the set $\mathcal{X}$. Distance and projection are always defined in a Euclidean sense.

\subsection{Gradient Descent}
The theory of gradient descent often takes its point of departure to be
the study of convex minimization where the function is both $\ell$-smooth 
and $\alpha$-strongly convex:
\begin{definition}\label{def:smooth}
A differentiable function $f(\cdot)$ is \textbf{$\ell$-smooth (or $\ell$-gradient Lipschitz)} if:
\begin{equation*}
\forall \x_1, \x_2, ~\norm{\grad f(\x_1) - \grad f(\x_2)} \le \ell \norm{\x_1 - \x_2}.
\end{equation*}
\end{definition}
\begin{definition}
A twice-differentiable function $f(\cdot)$ is \textbf{$\alpha$-strongly convex} if
$\forall \x, ~\lambda_{\min}(\hess f(\x)) \ge \alpha$
\end{definition}
Such smoothness guarantees imply that the gradient can not change 
too rapidly, and strong convexity ensures that there is a unique 
stationary point (and hence a global minimum). 
% Strongly convexity can also be defined without assuming twice-differentiability \cite{bubeck2015convex}. \praneeth{Is the above remark necessary?} \jccomment{Not sure, I added it in case optimization people are not happy that we only defined a special case, we can remove it if you feel OK.}
Standard analysis using these two properties shows that gradient descent converges linearly to a global optimum $\x^\star$ (see e.g. \citep{bubeck2015convex}).

\begin{theorem}\label{thm:strongly_convex}
Assume $f(\cdot)$ is $\ell$-smooth and $\alpha$-strongly convex. For any $\epsilon>0$, if we run gradient descent with step size $\eta = \frac{1}{\ell}$, iterate $\x_t$ will be $\epsilon$-close to $\x^\star$ in iterations:
% \vskip -0.25in
\begin{equation*}
\frac{2\ell}{\alpha}\log \frac{\norm{\x_0 - \x^\star}}{\epsilon}
\end{equation*}
\end{theorem}

% \jccomment{Maybe change this citation? This can not be the first one.}

In a more general setting, we no longer have convexity, let alone strong convexity. Though global optima are difficult to achieve in such a setting, it is possible to analyze convergence to first-order stationary points. %On the other hand, we note gradient descent will not converge as long as $\norm{\grad f(\x_t)} \neq 0$. Standard analysis then analyze the convergence to first-order stationary point. 
\begin{definition}
For a differentiable function $f(\cdot)$, we say that $\x$ is a \textbf{first-order stationary point} if $\norm{\grad f(\x)} =0$; we also say $\x$ is an \textbf{\EFSP} if $\norm{\grad f(\x)} \le \epsilon$.
\end{definition}

% Under $\ell$-smoothness assumption, we can still write down Taylor expansion up to second order at iterate $\x_t$:
% \begin{equation*}
% f(\x) \le f(\x_t) + \la \grad f(\x_t), \x - \x_t\ra + \frac{\ell}{2} \norm{\x - \x_t}^2
% \end{equation*}
% A natural algorithm is to pick $\x_{t+1}$ by minimizing the right hand side which gives gradient descent with natural choice of step size $\eta = \frac{1}{\ell}$.
% It is well known that this choice of step size gives convergence of gradient descent to first-order stationary points.

Under an $\ell$-smoothness assumption, it is well known that by choosing the step size $\eta = \frac{1}{\ell}$, gradient descent converges to first-order stationary points.

\begin{theorem}[\citep{nesterov1998introductory}]\label{thm:grad_smooth}
Assume that the function $f(\cdot)$ is $\ell$-smooth. Then, for any $\epsilon>0$, if we run gradient descent with step size $\eta = \frac{1}{\ell}$ and termination condition $\norm{\grad f(\x)} \le \epsilon$, the output will be \EFSP, and the algorithm will terminate within the following number of iterations:
\begin{equation*}
\frac{\ell(f(\x_0) - f^\star)}{\epsilon^2}.
\end{equation*}
\end{theorem}

% \jccomment{Do we need to check this citation? It's a long book, I get this citation from Elad's paper}

Note that the iteration complexity does not depend explicitly on intrinsic dimension; in the literature this is referred to as ``dimension-free optimization.''

A first-order stationary point can be either a local minimum or a saddle point or a local maximum. For minimization problems, saddle points and local maxima are undesirable, and we abuse nomenclature to call both of them ``saddle points'' in this paper. The formal definition is as follows:
\begin{definition}
For a differentiable function $f(\cdot)$, we say that $\x$ is a \textbf{local minimum} if $\x$ is a first-order stationary point, and there exists $\epsilon>0$ so that for any $\y$ in the $\epsilon$-neighborhood of $\x$, we have $f(\x) \le f(\y)$; we also say $\x$ is a \textbf{saddle point} if $\x$ is a first-order stationary point but not a local minimum. For a twice-differentiable function $f(\cdot)$, we further say a saddle point $\x$ is \textbf{strict} (or \textbf{non-degenerate}) if $\lambda_{\min}(\hess f(\x))<0$.
\end{definition}

For a twice-differentiable function $f(\cdot)$, we know a saddle point $\x$ must satify $\lambda_{\min}(\hess f(\x))\le 0$. Intuitively, for saddle point $\x$ to be strict, we simply rule out the undetermined case $\lambda_{\min}(\hess f(\x))=0$, where Hessian information alone is not enough to check whether $\x$ is a local minimum or saddle point. In most non-convex problems, saddle points are undesirable.

To escape from saddle points and find local minima in a general
setting, we move both the assumptions and guarantees in 
Theorem \ref{thm:grad_smooth} one order higher. In particular, we 
require the Hessian to be Lipschitz:
% in addition to gradient.
\begin{definition}\label{def:HessianLip}
A twice-differentiable function $f(\cdot)$ is \textbf{$\rho$-Hessian Lipschitz} if:
\begin{equation*}
\forall \x_1, \x_2, ~\norm{\hess f(\x_1) - \hess f(\x_2)} \le \rho \norm{\x_1 - \x_2}.
\end{equation*}
\end{definition}
That is, Hessian can not change dramatically in terms of spectral norm. 
We also generalize the definition of first-order stationary point to higher order:
\begin{definition}
For a $\rho$-Hessian Lipschitz function $f(\cdot)$, we say that $\x$ is a \textbf{second-order stationary point} if $\norm{\grad f(\x)} = 0$ and $\lambda_{\min}(\hess f(\x))\ge0 $;
we also say $\x$ is \textbf{\ESSP} if:
% $\norm{\grad f(\x)} \le \epsilon$ and $\lambda_{\min}(\hess f(\x)) \ge - \sqrt{\rho \epsilon}$.
\begin{equation*}
\norm{\grad f(\x)} \le \epsilon, \quad\text{and}\quad \lambda_{\min}(\hess f(\x)) \ge - \sqrt{\rho \epsilon}
\end{equation*}
\end{definition}
Second-order stationary points are very important in non-convex optimization because when all saddle points are strict, all second-order stationary points are exactly local minima.
 % and hence convergence to second-order stationary points is equivalent to convergence to local minima.

Note that the literature sometime defines \ESSP ~by two independent error
terms; i.e., letting $\norm{\grad f(\x)} \le \epsilon_g$ and $\lambda_{\min}(\hess f(\x)) \ge - \epsilon_H$. We instead follow the convention of \citet{nesterov2006cubic} by choosing $\epsilon_H = \sqrt{\rho \epsilon_g}$ to reflect the natural relations between the gradient and the Hessian.
This definition of \ESSP ~can also differ by reparametrization (and scaling), e.g. \citet{nesterov2006cubic} use $\epsilon' = \sqrt{\epsilon/\rho}$. We choose our parametrization so that the first requirement of \ESSP ~coincides with the requirement of \EFSP, for a fair comparison of our result with Theorem \ref{thm:grad_smooth}.
% Following the convention of \cite{nesterov2006cubic}.

% \jccomment{Comment this choice of $\sqrt{\rho\epsilon}$ because gradient and Hessian has different unit}
% \jccomment{Comment on why this choice of polynomial dependence of epsilon is nature}

%!TEX root = main.tex

\section{Main Result}
%We now turn to our main results. We will state our algorithms, assumptions, formal theorems, their extensions and implications.

In this section we show that it possible to modify gradient descent in
a simple way so that the resulting algorithm will provably converge 
quickly to a second-order stationary point.

The algorithm that we analyze is a perturbed form of gradient descent 
(see Algorithm \ref{algo:PGD}). The algorithm is based on gradient 
descent with step size $\eta$. When the norm of the current gradient is small 
($\le g_{\text{thres}}$) (which indicates that the current iterate 
$\tilde{\x}_{t}$ is potentially near a saddle point), the algorithm 
adds a small random perturbation to the gradient. The perturbation 
is added at most only once every $t_{\text{thres}}$ iterations.

To simplify the analysis we choose the perturbation $\xi_t$ to be uniformly sampled from a $d$-dimensional ball\footnote{Note that uniform sampling from a $d$-dimensional ball can be done efficiently by sampling $U^{\frac{1}{d}}\times \frac{\Y}{\norm{\Y}}$ where $U \sim \text{Uniform}([0, 1])$ and $\Y \sim \mathcal{N}(0, \I_d)$ \citep{harman2010decompositional}.}. %, which helps better explore the structure around saddle points.
The use of the threshold $t_{\text{thres}}$ ensures that the dynamics
are mostly those of gradient descent.  If the function value does not 
decrease enough (by $f_{\text{thres}}$) after $t_{\text{thres}}$ iterations, the algorithm outputs $\tilde{\x}_{t_{\text{noise}}}$. The analysis in
this section shows that under this protocol, the output 
$\tilde{\x}_{t_{\text{noise}}}$ is necessarily ``close'' to 
a second-order stationary point.
%\jccomment{Should I say poor saddle point instead of saddle point? briefly mention what is poor saddle point? also in later context}

\begin{algorithm}[t]
\caption{Perturbed Gradient Descent: $\text{PGD}(\x_0, \ell, \rho, \epsilon, c, \delta, \Delta_f)$}\label{algo:PGD}
\begin{algorithmic}
\STATE $\chi \leftarrow 3\max\{\log(\frac{d\ell\Delta_f}{c\epsilon^2\delta}), 4\}, 
 ~\eta \leftarrow \frac{c}{\ell},
 ~r \leftarrow \frac{\sqrt{c}}{\chi^2}\cdot\frac{\epsilon}{\ell}, 
 ~g_{\text{thres}} \leftarrow \frac{\sqrt{c}}{\chi^2}\cdot \epsilon,
 ~f_{\text{thres}} \leftarrow \frac{c}{\chi^3} \cdot \sqrt{\frac{\epsilon^3}{\rho}}, 
 ~t_{\text{thres}} \leftarrow \frac{\chi}{c^2}\cdot\frac{\ell}{\sqrt{\rho \epsilon}}$
\STATE $t_{\text{noise}} \leftarrow -t_{\text{thres}}-1$
\FOR{$t = 0, 1, \ldots $}
\IF{$\norm{\grad f(\x_t)} \le g_{\text{thres}}$ and $t - t_{\text{noise}} > t_{\text{thres}}$} 
\STATE $\tilde{\x}_t \leftarrow \x_t, \quad t_{\text{noise}} \leftarrow t$
\STATE $\x_t \leftarrow \tilde{\x}_t + \xi_t, \qquad \xi_t \text{~uniformly~} \sim \ball_0(r)$
\ENDIF
\IF{$t - t_{\text{noise}} = t_{\text{thres}}$ and $f(\x_t) - f(\tilde{\x}_{t_{\text{noise}}}) > -f_{\text{thres}}$}
\STATE \textbf{return} $\tilde{\x}_{t_{\text{noise}}}$
\ENDIF
\STATE $\x_{t+1} \leftarrow \x_t - \eta \grad f (\x_t)$
\ENDFOR
\end{algorithmic}
\end{algorithm}

%Our main theoretical result establishes the convergence of this perturbed gradient descent algorithm to a second-order stationary point. This greatly improved upon traditional result Theorem \ref{thm:grad_smooth} which only guarantees convergence to first-order stationary point.
We first state the assumptions that we require.
\begin{assumption}\label{as:smooth_Lip}
Function $f(\cdot)$ is both $\ell$-smooth and $\rho$-Hessian Lipschitz.
\end{assumption}
%Hessian Lipschitz condition guarantees around saddle point, the negative eigendirection of Hessian can not change too rapidly. This allows gradient descent to follow this slowly changed descent direction, thus escaping from saddle points. R: I feel these are not very clear.

The Hessian Lipschitz condition ensures that the function is well-behaved near a saddle point, and the small perturbation we add will suffice to allow the 
subsequent gradient updates to escape from the saddle point. More formally, we have:
\begin{theorem}
\label{thm:main}
Assume that $f(\cdot)$ satisfies \ref{as:smooth_Lip}.  Then there 
exists an absolute constant $c_{\max}$ such that, for any $\delta>0, 
\epsilon \le \frac{\ell^2}{\rho}$, $\Delta_f \ge f(\x_0) - f^\star$, 
and constant $c \le c_{\max}$, $\text{PGD}(\x_0, \ell, \rho, \epsilon, c, \delta, \Delta_f)$ will output an \ESSP, with probability $1-\delta$, and 
terminate in the following number of iterations:
\begin{equation*}
O\left(\frac{\ell(f(\x_0) - f^\star)}{\epsilon^2}\log^{4}\left(\frac{d\ell\Delta_f}{\epsilon^2\delta}\right) \right).
\end{equation*}
\end{theorem}
%\noindent\textbf{Remarks}:
%\begin{itemize}
%\item 
%%\item Compared with Theorem \ref{thm:grad_smooth}, our iteration complexity to find an \ESSP ~is  the same as the well-known complexity of gradient descent to find an \EFSP ~except $O(\log^{4}(\frac{d\ell\Delta_f}{\epsilon^2\delta}))$ factors. R: I feel we have emphasized this point enough number of times and no need to emphasize it again here.
%% Our iteration complexity is the same as the one for standard gradient descent to find a stationary point except log factors.
%\item Note in Algorithm \ref{algo:PGD}, our choice of learning rate is $O(1/\ell)$ which is the same as the natural choice of learning rate in traditional smooth analysis.
%\item We believe the constant $c$ in the choice of $\eta, r, g_{\text{thres}}, f_{\text{thres}}, t_{\text{thres}}$ in Algorithm \ref{algo:PGD} can be indepedently chosen within certain range, and need not to be tied to each other. For simplicity of presentation and proof, we minimize the number of abosolute constants, and use the same $c$.
%\end{itemize}

% \textcolor{blue}{R: I feel the remarks are either too minor (and distracting) or already emphasized many times...} \jccomment{I think current style (putting remark content into paragraphs also look good)}

Strikingly, Theorem~\ref{thm:main} shows that perturbed gradient descent finds a second-order stationary point in almost the same amount of time that gradient descent takes to find first-order stationary point. The step size $\eta$ is chosen as $O(1/\ell)$ which is in accord with classical analyses of convergence
to first-order stationary points. Though we state the theorem with a certain choice of parameters for simplicity of presentation, our result holds even if we vary the parameters up to constant factors.

Without loss of generality, we can focus on the case 
$\epsilon \le \ell^2/\rho$, as in Theorem \ref{thm:main}. 
This is because in the case $\epsilon>\ell^2/\rho$, standard gradient descent 
without perturbation---Theorem \ref{thm:grad_smooth}---easily 
solves the problem (since by \ref{as:smooth_Lip}, 
we always have $\lambda_{\min}(\hess f(\x)) \ge -\ell \ge - \sqrt{\rho \epsilon}$, which means that all \ESSP s are $\epsilon$-first order stationary points).

We believe that the dependence on at least one $\log d$ factor in 
the iteration complexity is unavoidable in the non-convex setting, 
as our result can be directly applied to the principal component 
analysis problem, for which the best known runtimes (for the power 
method or Lanczos method) incur a $\log d$ factor. Establishing
this formally is still an open question however.
% We conjecture similar phenomena may hold true for many other popular first order algorithms as well.
% \jccomment{Maybe adds some other explanation about Main theorem before discussing the intuition why it's true? I can't think of other words to talk here.}

To provide some intuition for Theorem \ref{thm:main}, consider an iterate $\x_t$ which is not yet an \ESSP. By definition, either (1) the gradient $\grad f(\x_t)$ is large, or (2) the Hessian $\hess f(\x_t)$ has a significant negative eigenvalue. Traditional analysis works in the first case. The crucial step in the proof of Theorem \ref{thm:main} involves handling the second case: when the gradient is small $\norm{\nabla f(\x_t)} \le g_{\text{thres}}$ and the Hessian has a significant negative eigenvalue $\lambda_{\min}(\hess f(\tilde{\x}_t)) \le -\sqrt{\rho\epsilon}$,
then adding a perturbation, followed by standard gradient descent 
for $t_{\text{thres}}$ steps, decreases the function value by at 
least $f_{\text{thres}}$, with high probability.  The proof of 
this fact relies on a novel characterization of geometry 
around saddle points (see Section \ref{sec:sketch})
 % for details). 

If we are able to make stronger assumptions on the objective function 
we are able to strengthen our main result.  This further analysis is
presented in the next section.

\subsection{Functions with Strict Saddle Property}
In many real applications, objective functions further admit the property that all saddle points are strict 
% (i.e. for any saddle point $\x_s$, $\lambda_{\min}(\hess f(\x_s))<0$) 
\citep{ge2015escaping,sun2016complete,sun2016geometric,bhojanapalli2016global,ge2016matrix}. 
% This implies that higher-order saddle points do not exist for these problems. 
In this case, all second-order stationary points are local minima and hence convergence to second-order stationary points (Theorem \ref{thm:main}) is equivalent to convergence to local minima.

To state this result formally,
% it is not enough to only assume the property exactly at saddle point (which in most cases is measure zero). 
we introduce a robust version of the strict saddle 
property~\citep[cf.][]{ge2015escaping}:
% When we assume strict saddle condition, this theorem directly gives convergence result to local minima.
% Comments on many other applications has these properties.
\begin{assumption}\label{as:strict_saddle}
Function $f(\cdot)$ is $(\theta, \gamma, \zeta)$-\textbf{strict saddle}. That is, for any $\x$, at least one of following holds:
\begin{itemize}
\item $\norm{\grad f(\x)} \ge \theta$.
\item $\lambda_{\min}(\hess f(\x)) \le -\gamma$.
\item $\x$ is $\zeta$-close to $\cXstar$ --- the set of local minima.
\end{itemize}
\end{assumption}
Intuitively, the strict saddle assumption states that the $\R^d$ space can be divided into three regions: 1) a region where the gradient is large; 2) a region
where the Hessian has a significant negative eigenvalue (around saddle point); 
and 3) the region close to a local minimum. With this assumption, we immediately have the following corollary:

\begin{corollary}\label{cor:main_strictsaddle}
Let $f(\cdot)$ satisfy \ref{as:smooth_Lip} and \ref{as:strict_saddle}. 
Then, there exists an absolute constant $c_{\max}$ such that, 
for any $\delta>0, \Delta_f \ge f(\x_0) - f^\star$, constant 
$c \le c_{\max}$, and letting $\tilde{\epsilon} = \min(\theta, \gamma^2/\rho)$, 
$\text{PGD}(\x_0, \ell, \rho, \tilde{\epsilon}, c, \delta, \Delta_f)$ 
will output a point $\zeta$-close to $\cXstar$, with probability $1-\delta$, 
and terminate in the following number of iterations:
\begin{equation*}
O\left(\frac{\ell(f(\x_0) - f^\star)}{\tilde{\epsilon}^2}\log^{4}\left(\frac{d\ell\Delta_f}{\tilde{\epsilon}^2\delta}\right)  \right).
\end{equation*}
\end{corollary}

% \jccomment{Revise here}
Corollary~\ref{cor:main_strictsaddle} shows that by substituting 
$\epsilon$ in Theorem \ref{thm:main} using $\tilde{\epsilon} = \min(\theta, \gamma^2/\rho)$, the 
output of perturbed gradient descent will be in the $\zeta$-neighborhood of some local minimum.
% that by substituting 
% $\epsilon$ in Theorem \ref{thm:main}, and letting $\tilde{\epsilon} = \min(\theta, \gamma^2/\rho)$, the output of perturbed gradient descent will 
% necessarily be found in the $\zeta$-neighborhood of some local minimum.
% This is true because by definition of $(\theta, \gamma, \zeta)$-strict 
% saddle, an $\tilde{\epsilon}$-second-order stationary point is guaranteed to be $\zeta$ close to some local minimum.

% Corollary~\ref{cor:main_strictsaddle} shows 
% after finding $\tilde{\epsilon}$-second-order stationary point
% by Theorem \ref{thm:main} where $\tilde{\epsilon} = \min(\theta, \gamma^2/\rho)$, 
% the output is also in the $\zeta$-neighborhood of some local minimum.

Note although Corollary \ref{cor:main_strictsaddle} only explicitly asserts that the output will lie within some fixed radius $\zeta$ from a local minimum.  In many real applications, we can further write $\zeta$ as a function $\zeta(\cdot)$ of gradient threshold $\theta$, so that when $\theta$ decreases, $\zeta(\theta)$ decreases linearly or polynomially depending on $\theta$. Meanwhile, parameter $\gamma$ is always nondecreasing when $\theta$ decreases due to the nature of this strict saddle definition. Therefore, in these cases, the above corollary further gives a convergence rate to a local minimum.

% , from the definition of $(\theta, \gamma, \zeta)$-strict saddle, we know that if we reduce the gradient threshold $\theta$, points where gradient is smaller than $\theta$ will be still closer to saddle points or local minima, thus $\gamma$ will be non-decreasing and $\zeta$ will be non-increasing.

% \jccomment{Comment in many cases, $\zeta$ has linear / polynomical dependence on $\theta$}

% \begin{corollary}
% If function $f(\cdot)$ satisfies \ref{as:Hessian_Lip}, \ref{as:smooth_Lip}, \ref{as:strict_saddle}.
% Let $\tilde{\epsilon} = \min(\theta, \gamma^2/\rho)$ and we run nonconvex gradient descent (Algorithm \ref{algo:nonconvex_GD}) with parameter chosen as Eq.~\eqref{eq:parameter_choice} by replacing $\epsilon$ with $\tilde{\epsilon}$. Then the output will be $\Delta-$close to some local minima with at least $1-\delta$ probability in iterations:
% \begin{equation*}
% % O\left(\frac{\ell(f(\x_0) - f^*)}{\tilde{\epsilon}^2} \log^4 (d\cn)\right)
% O\left(\frac{\ell(f(\x_0) - f^\star)}{\tilde{\epsilon}^2}\log^{4}(\frac{d \ell\rho(f(\x_0) - f^\star)}{\tilde{\epsilon} \delta}) \right)
% \end{equation*}
% \end{corollary}

% ~

% This Corollary is immediately implied by Theorem \ref{thm:main}.
% Comment on only replacing $\epsilon$ with $\tilde{\epsilon}$

\subsection{Functions with Strong Local Structure}

The convergence rate in Theorem \ref{thm:main} is polynomial in $\epsilon$, which is similar to that of Theorem~\ref{thm:grad_smooth}, but is worse than the
rate of Theorem~\ref{thm:strongly_convex} because of the lack of strong convexity.
Although global strong convexity does not hold in the non-convex setting
that is our focus, in many machine learning problems the objective 
function may have a favorable local structure in the neighborhood of 
local minima \citep{ge2015escaping,sun2016complete,sun2016geometric,sun2016guaranteed}. Exploiting this property can lead to much faster convergence (linear convergence) to local minima. One such property that ensures such
convergence is a local form of smoothness and strong convexity:
\edef\oldassumption{\the\numexpr\value{assumption}+1}
\setcounter{assumption}{0}
\renewcommand{\theassumption}{A\oldassumption.\alph{assumption}}
\begin{assumption}\label{as:sc}
In a $\zeta$-neighborhood of the set of local minima $\cXstar$, 
the function $f(\cdot)$ is $\alpha$-\textbf{strongly convex}, 
and $\beta$-smooth.
\end{assumption}
Here we use different letter $\beta$ to denote the local smoothness parameter 
(in contrast to the global smoothness parameter $\ell$). Note that we always 
have $\beta \le \ell$. 
% Sometimes $\beta$ can be much smaller than $\ell$ which allows larger step size locally, and faster convergence.
However, often even local $\alpha$-strong convexity does not hold. 
We thus introduce the following relaxation:
\begin{assumption}\label{as:regularity}
In a $\zeta$-neighborhood of the set of local minima $\cXstar$, 
the function $f(\cdot)$ satisfies a $(\alpha, \beta)$-\textbf{regularity condition} if for any $\x$ in this neighborhood:
\begin{equation}\label{eq:regularity}
\la \grad f(\x), \x - \projX(\x) \ra \ge \frac{\alpha}{2}\norm{\x- \projX(\x)}^2 + \frac{1}{2\beta} \norm{\grad f(\x)}^2.
\end{equation} 
\end{assumption}

% \jccomment{Be more precise about how we define projection and how to deal with case of draw.}
Here $\projX(\cdot)$ is the projection on to the set $\cXstar$. Note $(\alpha, \beta)$-regularity condition is more general and is directly implied by standard $\beta$-smooth and $\alpha$-strongly convex conditions. 
% In general, we can also prove that this regularity condition is weaker and more general than strong convexity and smoothness around local minima. 
This regularity condition commonly appears in low-rank problems such as matrix sensing and matrix completion, and has been used in \citet{bhojanapalli2016global,zheng2016convergence},
where local minima form a connected set, and where the Hessian is strictly positive only with respect to directions pointing outside the set of local minima.
% \jccomment{Explain along the direction of other local min, Hessian must be strictly positive}

\begin{algorithm}[t]
\caption{Perturbed Gradient Descent with Local Improvement: $\text{PGDli}(\x_0, \ell, \rho, \epsilon, c, \delta, \Delta_f, \beta)$}\label{algo:PGDli}
\begin{algorithmic}
\STATE $\x_0 \leftarrow \text{PGD}(\x_0, \ell, \rho, \epsilon, c, \delta, \Delta_f)$
\FOR{$t = 0, 1, \ldots $}
\STATE $\x_{t+1} \leftarrow \x_t - \frac{1}{\beta} \grad f (\x_t)$
\ENDFOR
\end{algorithmic}
\end{algorithm}

Gradient descent naturally exploits local structure very well. In Algorithm \ref{algo:PGDli}, we first run Algorithm \ref{algo:PGD} to output a point within 
the neighborhood of a local minimum, and then perform standard gradient descent with step size $\frac{1}{\beta}$. We can then prove the following theorem:
\begin{theorem}\label{thm:main_local}
Let $f(\cdot)$ satisfy \ref{as:smooth_Lip}, \ref{as:strict_saddle}, and \ref{as:sc} (or \ref{as:regularity}). 
Then there exists an absolute constant $c_{\max}$ such that, 
for any $\delta>0, \epsilon>0, \Delta_f \ge f(\x_0) -f^\star$, 
constant $c \le c_{\max}$, and letting $\tilde{\epsilon} = 
\min(\theta, \gamma^2/\rho)$, $\text{PGDli}(\x_0, \ell, \rho, \tilde{\epsilon}, c, \delta, \Delta_f, \beta)$ will output a point that is 
$\epsilon$-close to $\cXstar$, with probability $1-\delta$, 
in the following number of iterations:
\begin{equation*}
O\left(\frac{\ell(f(\x_0) - f^\star)}{\tilde{\epsilon}^2}\log^{4}\left(\frac{d\ell\Delta_f}{\tilde{\epsilon}^2\delta}\right)  + \frac{\beta}{\alpha}\log \frac{\zeta}{\epsilon}\right).
\end{equation*}
\end{theorem}
Theorem \ref{thm:main_local} says that if strong local structure is present, the convergence rate can be boosted to linear convergence ($\log\frac{1}{\epsilon}$).
% compared to the polynomial convergence of Theorem \ref{thm:main} and Corollary \ref{cor:main_strictsaddle}.
In this theorem we see that sequence of iterations can be decomposed 
into two phases. In the first phase, perturbed gradient descent finds 
a $\zeta$-neighborhood by Corollary \ref{cor:main_strictsaddle}. 
In the second phase, standard gradient descent takes us from 
$\zeta$ to $\epsilon$-close to a local minimum. 
Standard gradient descent and Assumption \ref{as:sc} (or \ref{as:regularity}) make sure that the iterate never steps out of a $\zeta$-neighborhood in this second phase, giving a result similar to Theorem \ref{thm:strongly_convex} with linear convergence. 
% These conclude the key ideas behind Theorem \ref{thm:main_local}.

Finally, we note our choice of local conditions (Assumption \ref{as:sc} and \ref{as:regularity}) are not special. The interested reader can refer to \citet{karimi2016linear} for other relaxed and alternative notions of convexity, which can also be potentially combined with Assumptions $\ref{as:smooth_Lip} and \ref{as:strict_saddle}$ to yield convergence results of a similar flavor as that of Theorem \ref{thm:main_local}.

% . As long as the local assumption guarantees convergence to local minimum inside the local region, it should be fairly straightforward to combine that local condition with Assumption $\ref{as:smooth_Lip}, \ref{as:strict_saddle}$ to obtain convergence result of a similar flavor as that of Theorem \ref{thm:main_local}. On the other hand, weaker assumptions such as local Polyak-Lojasiewicz condition \citep{polyak1963gradient} only guarantee that function value decreases per step in the local region, and not necessarily that iterates converge to local minima. This type of constraint may cause a problem around poor local minima (whose function value is much higher than global minima) as gradient descent may leave the neighborhood of that poor local minima. However, if it is the case that all local minima are global minima, we believe Polyak-Lojasiewicz condition and Assumption $\ref{as:smooth_Lip}, \ref{as:strict_saddle}$ should also guarantee convergence to global minima.
% \praneeth{Is this a conjecture or we know it is true?} \jccomment{I think with mild conditions, I know it's true, but we didn't prove it.}\praneeth{Then this sentence sounds good.}

% \jccomment{Comment on possibly use PL condition, but here we converge to local minima, PL only allows convergence to global minima}

\section{Example --- Matrix Factorization}
As a simple example to illustrate how to apply our general theorems to specific non-convex optimization problems, we consider a symmetric low-rank matrix factorization problem, based on the following objective function:
\begin{equation}\label{eq:obj}
\min_{\U \in \R^{d\times r}} f(\U) = \frac{1}{2}\fnorm{\U\U\trans - \M^\star}^2,
\end{equation} 
where $\M^\star \in \R^{d\times d}$.  For simplicity, we assume 
$\rank(\M^\star) = r$, and denote $\sigstarl \defeq \sigma_1(\M^\star)$, $\sigstarr \defeq \sigma_r(\M^\star)$.
Clearly, in this case the global minimum of function value is zero, which is achieved at $\V^\star = \T\D^{1/2}$ where $\T\D\T\trans$ is the SVD of the symmetric real matrix $\M^\star$.

% In order to apply Theorem \ref{thm:main_local}, we can first vectorize $\U \in \R^{d\times r}$ to be a vector with $dr$ dimension, and then verify all the geometric assumptions (\ref{as:smooth_Lip}, \ref{as:strict_saddle}, \ref{as:regularity}) about objective function hold:

The following two lemmas show that the objective function in Eq.~\eqref{eq:obj} satisfies the geometric assumptions \ref{as:smooth_Lip}, \ref{as:strict_saddle},and \ref{as:regularity}.  Moreover, all local minima are global minima.

\begin{lemma}\label{lem:mf_smooth}
For any $\Gamma \ge \sigstarl$, the function $f(\U)$ defined in 
Eq.~\eqref{eq:obj} is $8\Gamma$-smooth and $12\Gamma^{1/2}$-Hessian Lipschitz,
inside the region $\{\U| \norm{\U}^2 < \Gamma \}$.
\end{lemma}

% \begin{lemma}\label{lem:mf_strictsaddle}
% For function $f(\U)$ defined in Eq.~\eqref{eq:obj}, all local minima are global minima. The set of global minima is $\cXstar = \{\V^\star \mR | \mR\mR\trans=\mR\trans\mR = \I \}$. Furthermore, $f(\U)$ is $(\frac{1}{24}(\sigstarr)^{3/2}, \frac{1}{3}\sigstarr, \frac{1}{3}(\sigstarr)^{1/2})$-strict saddle; and satisfies a
% $(\frac{2}{3}\sigstarr, 10\sigstarl)$-regularity condition in a $\frac{1}{3}(\sigstarr)^{1/2}$-neighborhood of $\cXstar$.
% \end{lemma} 

\begin{lemma}\label{lem:mf_strictsaddle}
For function $f(\U)$ defined in Eq.\eqref{eq:obj}, all local minima are global minima. The set of global minima is $\cXstar = \{\V^\star \mR | \mR\mR\trans=\mR\trans\mR = \I \}$. Furthermore, $f(\U)$ satisfies:
\begin{enumerate}
\item $(\frac{1}{24}(\sigstarr)^{3/2}, \frac{1}{3}\sigstarr, \frac{1}{3}(\sigstarr)^{1/2})$-strict saddle property.
\item $(\frac{2}{3}\sigstarr, 10\sigstarl)$-regularity condition in $\frac{1}{3}(\sigstarr)^{1/2}$ neighborhood of $\cXstar$.
\end{enumerate}
\end{lemma} 
% \cite{bhojanapalli2016global}

% Lemma \ref{lem:mf_smooth} and Lemma \ref{lem:mf_strictsaddle} tell us that the objective function of matrix factorization is smooth and Hessian Lipschitz in most regions; strict saddle globally; and regular locally in the neighborhood of local minima. Furthermore, Lemma \ref{lem:mf_strictsaddle} claims all local minima in matrix factorization are global minima. Therefore, finding local minima suffices to solve this problem, we already have roughly all of the preconditions for Theorem \ref{thm:main_local}. One caveat is that since the objective function is actually fourth-order polynomial with respect to $\U$, the smoothness and Hessian Lipschitz parameters from Lemma~\ref{lem:mf_smooth} become very large when the spectral norm of $\U$ is very large. In order to use tight smoothness and Hessian Lipschitz parameters, fortunately, we can further show that gradient descent (even with perturbation) does not increase the spectral norm of $\U$ beyond some limit, given $\norm{\U_0}$ is not too large. Our formal statement is as follows:
 
One caveat is that since the objective function is actually a fourth-order polynomial with respect to $\U$, the smoothness and Hessian Lipschitz parameters from Lemma~\ref{lem:mf_smooth} naturally depend on $\norm{\U}$.
Fortunately, we can further show that gradient descent (even with perturbation) does not increase $\norm{\U}$ beyond $O(\max\{\norm{\U_0}, (\sigstarl)^{1/2}\})$. Then, applying Theorem \ref{thm:main_local} gives:

\begin{theorem}\label{thm:mf_global}
There exists an absolute constant $c_{\max}$ such that the following holds. 
For the objective function in Eq.~\eqref{eq:obj}, for any $\delta >0$ and 
constant $c\le c_{\max}$, and for $\Gamma^{1/2} \defeq 2\max\{\norm{\U_0}, 
3(\sigstarl)^{1/2}\}$, the output of $\text{PGDli}(\U_0, 8\Gamma, 
12\Gamma^{1/2}, \frac{(\sigstarr)^{2}}{108\Gamma^{1/2}}, c, \delta, 
\frac{r\Gamma^2}{2}, 10\sigstarl)$, will be $\epsilon$-close to the 
global minimum set $\cXstar$, with probability $1-\delta$, after the
following number of iterations:
\begin{equation*}
O\left(r\left(\frac{ \Gamma}{\sigstarr}\right)^4\log^{4}\left(\frac{d \Gamma}{\delta\sigstarr}\right)  + \frac{\sigstarl}{\sigstarr}\log \frac{\sigstarr}{\epsilon} \right).
\end{equation*}
\end{theorem}
Theorem \ref{thm:mf_global} establishes global convergence of perturbed gradient descent from an arbitrary initial point $\U_0$, including exact saddle points. 
% However, since $\Gamma^{1/2} \defeq 2\max\{\norm{\U_0}, 3(\sigstarl)^{1/2}\}$, when initial spectral norm $\norm{\U_0}$ become larger, function become less smooth and Hessian Lipschitz, which requires smaller learning rate and longer time to converge. If we have some information about the scaling of $\M^\star$, and use a reasonal initialization where $\U_0 \le O((\sigstarl)^{1/2})$, 
Suppose we initialize at $\U_0 = 0$, then our iteration complexity becomes:
\begin{equation*}
O\left(r(\cn^\star)^4\log^{4}(d\cn^\star /\delta)  + \cn^\star \log(\sigstarr/\epsilon) \right),
\end{equation*}
where $\cn^\star = \sigstarl/\sigstarr$ is the condition number of the matrix $\M^\star$. 
We see that in the second phase, when convergence occurs inside the local region, we require $O(\cn^\star \log(\sigstarr/\epsilon))$ iterations which is the standard local linear rate for gradient descent. 
In the first phase, to find a neighborhood of the solution, 
our method requires a number of iterations scaling as $\tilde{O}(r(\cn^\star)^4)$. We suspect that this strong dependence on condition number arises
from our generic assumption that the Hessian Lipschitz is uniformly 
upper bounded; it may well be the case that this dependence can be
reduced in the special case of matrix factorization via a finer analysis
of the geometric structure of the problem.

\section{Proof Sketch for Theorem \ref{thm:main}} \label{sec:sketch}

In this section we will present the key ideas underlying the main result of this paper (Theorem~\ref{thm:main}). We will first argue the correctness of Theorem~\ref{thm:main} given two important intermediate lemmas. Then we turn
to the main lemma, which establishes that gradient descent can escape 
from saddle points quickly. We present full proofs of all these results in Appendix~\ref{app:main}.
Throughout this section, we use $\eta, r, g_{\text{thres}}, f_{\text{thres}}$ and $t_{\text{thres}}$ as defined in Algorithm~\ref{algo:PGD}.% on how they depends on parameters $\ell, \rho, \epsilon, c, \chi$.

\subsection{Exploiting Large Gradient or Negative Curvature}
%\jccomment{Change the title either here or in appendix}
%Recall theorem \ref{thm:main} made following two claims: 1) the output of $\text{PGD}(\x_0, \ell, \rho, \epsilon, c, \delta, \Delta_f)$ is an \ESSP; 2) the algorithm terminate in $\tilde{O}\left(\ell(f(\x_0) - f^\star)/\epsilon^2 \right)$ iterations.
%In order to prove these claims, first consider our objective is to find \ESSP. In proof, we actually find points with stronger properties as follows:
%In order to find an \ESSP, we need bounds on both gradient and Hessian. In fact, we will actually find points with stronger guarantees than Theorem~\ref{thm:main}:
%\begin{equation*}
% \norm{\grad f(\x)} \le g_{\text{thres}}= \sqrt{c}\epsilon/\chi^2, \qquad\text{~and~}\qquad \lambda_{\min}(\hess f(\x)) \ge - \sqrt{\rho \epsilon}
%\end{equation*}
%\textcolor{blue}{R: I'm not sure if we want to talk about this. Why is this relevant for a proof sketch? I don't think people care about how the parameters are chosen in proof sketch}
%
%Since $\chi \ge 1$, and in proof we let $c_{\max} \le 1$, thus $c\le c_{\max} \le 1$, we know $g_{\text{thres}} \le \epsilon$, thus any point satisfying above properties are automatically an \ESSP. 
Recall that an \ESSP~is a point with a small gradient, and where the Hessian does not have a significant negative eigenvalue. Suppose we are currently at an iterate $\x_t$ that is not an \ESSP; i.e., it does not satisfy the above 
properties.  There are two possibilities: 
% (1) The gradient is 
% large: $\norm{\grad f(\x_t)} \ge g_{\text{thres}}$; or (2) Around the
% saddle point we have $\norm{\grad f(\x_t)} \le g_{\text{thres}}$ and $\lambda_{\min}(\hess f(\x_t)) \le - \sqrt{\rho \epsilon}$.
\begin{enumerate}
\item Gradient is large: $\norm{\grad f(\x_t)} \ge g_{\text{thres}}$, or
\item Around saddle point: $\norm{\grad f(\x_t)} \le g_{\text{thres}}$ and $\lambda_{\min}(\hess f(\x_t)) \le - \sqrt{\rho \epsilon}$.
\end{enumerate}

The following two lemmas address these two cases respectively. They guarantee that perturbed gradient descent will decrease the function value in both scenarios.
\begin{lemma}[Gradient] \label{lem:sketch_gradient}
Assume that $f(\cdot)$ satisfies \ref{as:smooth_Lip}. Then for gradient descent with stepsize $\eta < \frac{1}{\ell}$, we have $f(\x_{t+1}) \le f(\x_t) - \frac{\eta}{2}\norm{\grad f(\x_t)}^2$.
\end{lemma}
\begin{lemma}[Saddle] (informal)  \label{lem:sketch_saddle}
Assume that $f(\cdot)$ satisfies ~\ref{as:smooth_Lip}, 
If $\x_t$ satisfies $\norm{\nabla f(\x_t)} \le g_{\text{thres}}$ and $\lambda_{\min}(\hess f(\x_t)) \le -\sqrt{\rho\epsilon}$, then
adding one perturbation step followed by $t_{\text{thres}}$ steps of gradient descent, we have
$f(\x_{t+t_{\text{thres}}}) - f(\x_t) \le -f_{\text{thres}}$ with high probability.
\end{lemma}

We see that Algorithm~\ref{algo:PGD} is designed so that Lemma \ref{lem:sketch_saddle} can be directly applied. According to these two lemmas, perturbed gradient descent will decrease the function value either in the case of a large gradient, or around strict saddle points. Computing the average decrease per step in function value yields the total iteration complexity. Since Algorithm~\ref{algo:PGD} only terminate when the function value decreases too slowly, this guarantees that the output must be \ESSP ~(see Appendix~\ref{app:main} for formal proofs).

\subsection{Main Lemma: Escaping from Saddle Points Quickly}
The proof of Lemma~\ref{lem:sketch_gradient} is straightforward and follows from traditional analysis. The key technical contribution of this paper is the proof of Lemma~\ref{lem:sketch_saddle}, which gives a new characterization
of the geometry around saddle points.

Consider a point $\tilde{\x}$ that satisfies the the preconditions of Lemma \ref{lem:sketch_saddle} ($\norm{\nabla f(\tilde{\x})} \le g_{\text{thres}}$ and $\lambda_{\min}(\hess f(\tilde{\x})) \le -\sqrt{\rho\epsilon}$). After adding the perturbation ($\x_0 = \tilde{\x} + \xi$), we can view $\x_0$ as coming from a uniform distribution over $\ball_{\tilde{\x}}(r)$, which we call the \textbf{perturbation ball}. We can divide this perturbation ball $\ball_{\tilde{\x}}(r)$ into two disjoint regions: (1) an \textbf{escaping region} $\cXe$ which consists of all the points $\x \in \ball_{\tilde{\x}}(r)$ whose function value decreases by at least $f_{\text{thres}}$ after $t_{\text{thres}}$ steps; % gradient descent starting from $\x$, the function value will decrease at least $f_{\text{thres}}$, thus escape the saddle point quickly; 
(2) a \textbf{stuck region} $\cXs = \ball_{\tilde{\x}}(r) - \cXe $.
% the set of points that do not have sufficient function value decrease after $t_{\text{thres}}$ steps.
% consisting of points which get stuck around saddle point. 
% Intuitively, all starting points in escaping region will escape the saddle point (by decreasing function value) rather quickly while all starting points in stuck region actually get stuck for at least reasonable amount of time. 
Our general proof strategy is to show that $\cXs$ 
% is a very ``thin'' high dimensional band, which 
consists of a very small proportion of the volume of perturbation ball. After adding a perturbation to $\tilde{\x}$, point $\x_0$ has a very small chance of falling in $\cXs$, and hence will escape from the saddle point efficiently.
% The major challenge here is what can we say about stuck region $\cXs$.

\begin{figure}
\centering
\begin{minipage}{.45\textwidth}
  \centering
  \includegraphics[trim={2cm 2cm 2cm 0cm}, width=\textwidth]{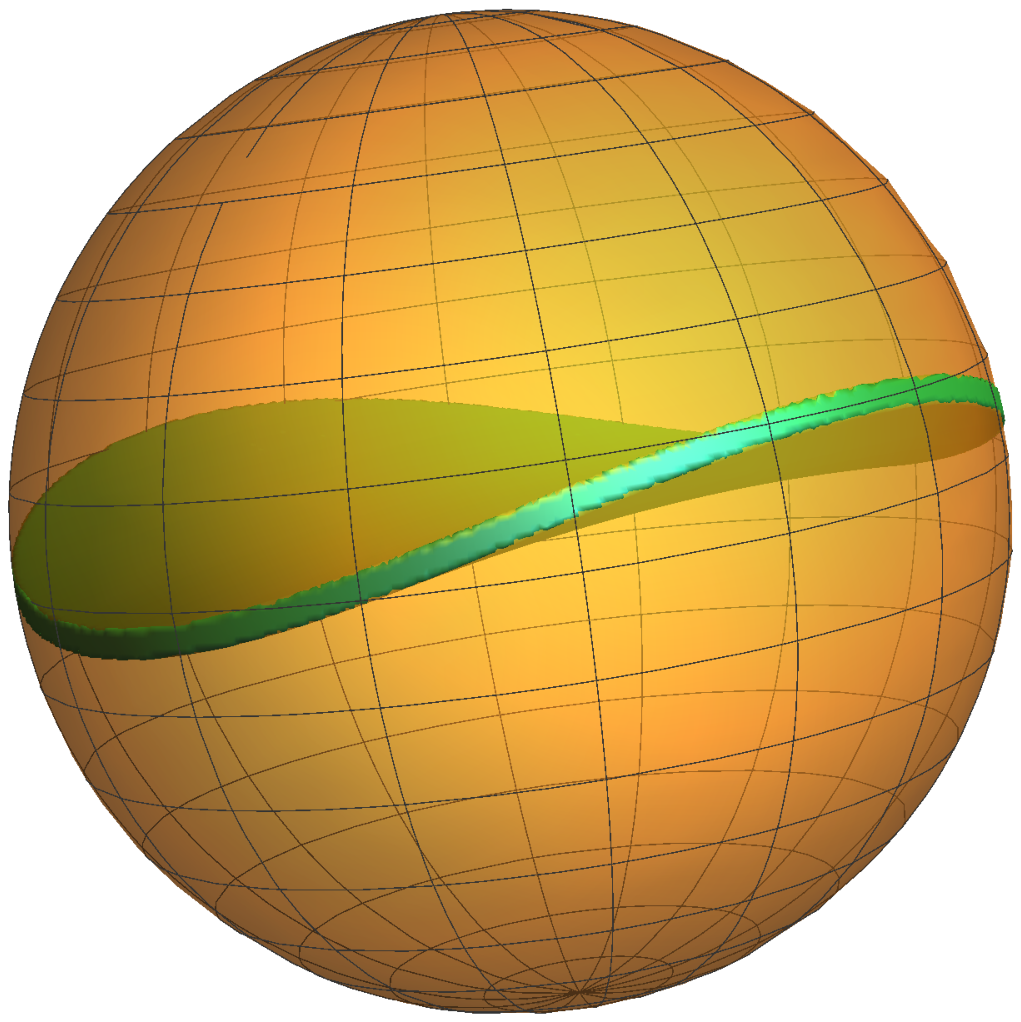}
  \captionof{figure}{Pertubation ball in 3D and ``thin pancake'' shape stuck region}
  \label{fig:band}
\end{minipage}%
\begin{minipage}{.05\textwidth}
~
\end{minipage}
\begin{minipage}{.45\textwidth}
  \centering
  \includegraphics[width=0.9\textwidth]{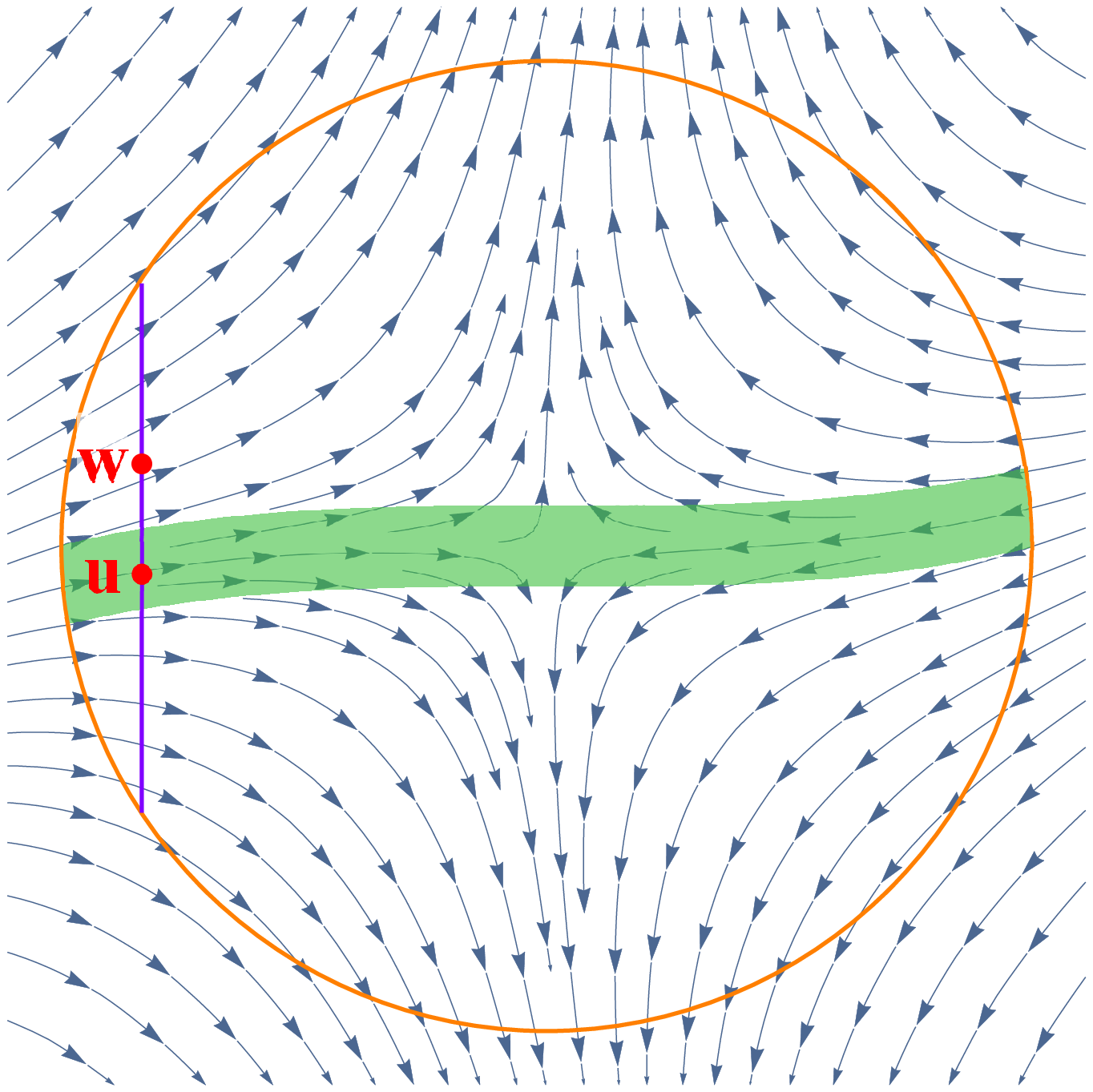}
  \captionof{figure}{Pertubation ball in 2D and ``narrow band'' stuck region under gradient flow}
  \label{fig:flow}
\end{minipage}
\end{figure}

% \begin{figure}[t]
% \vskip 0.2in
% \begin{center}
% \centerline{\includegraphics[trim={2cm 2cm 2cm 0cm}, width=0.55\columnwidth]{fig/band.eps}}
% \caption{Pertubation ball in 3D and ``thin pancake'' stuck region}
% \label{fig:band}
% \end{center}
% \begin{center}
% \centerline{\includegraphics[width=0.55\columnwidth]{fig/flow6.eps}}
% \caption{``Narrow band'' stuck region in 2D under gradient flow}
% \label{fig:flow}
% \end{center}
% \vskip -0.25in
% \end{figure} 

% \begin{figure}[ht]
% \vskip 0.2in

% \vskip -0.2in
% \end{figure} 

% \begin{figure}
% \centering
% \begin{minipage}{.45\textwidth}
%   \centering
%   \includegraphics[trim={2cm 2cm 2cm 0cm}, width=\textwidth]{fig/band.eps}
%   \captionof{figure}{Pertubation ball in 3D and ``thin pancake'' shape stuck region}
%   \label{fig:band}
% \end{minipage}%
% \begin{minipage}{.05\textwidth}
% ~
% \end{minipage}
% \begin{minipage}{.45\textwidth}
%   \centering
%   \includegraphics[width=0.9\textwidth]{fig/flow6.eps}
%   \captionof{figure}{Pertubation ball in 2D and ``narrow band'' stuck region under gradient flow}
%   \label{fig:flow}
% \end{minipage}
% \end{figure}

Let us consider the nature of $\cXs$. For simplicity, let us imagine that $\tilde{\x}$ is an exact saddle point whose Hessian has only one negative eigenvalue, and $d-1$ positive eigenvalues. Let us denote the minimum eigenvalue direction as $\e_1$.
% and assume the worst case scenario where all other eigenvalues are positive.
In this case, if the Hessian remains constant (and we have a quadratic function), the stuck region $\cXs$ consists of points $\x$ such that $\x - \tilde{\x}$ has a small $\e_1$ component. This is a straight band in two dimensions and 
a flat disk in high dimensions. However, when the Hessian is not constant, the shape of the stuck region is distorted. In two dimensions, it forms 
a ``narrow band'' as plotted in Figure~\ref{fig:flow} on top of
the gradient flow. In three dimensions, it forms a ``thin pancake'' 
as shown in  Figure~\ref{fig:band}. %\jccomment{Could I use pancake?}

The major challenge here is to bound the volume of this high-dimensional 
non-flat ``pancake'' shaped region $\cXs$.
% which in general we do not know how to describe these regions explicitly.
A crude approximation of this ``pancake" by a flat ``disk" loses 
polynomial factors in the dimensionalilty, which gives a suboptimal rate.
% To achieve sharp dimension dependence, our step size and perturbation must not polynomially depend on dimension $d$: trying to approximate the ``pancake" by scaling a flat ``disk" loses such factors since the ``pancake" is too distorted. This is the reason why existing works have polynomial dependence on $d$.
% In this case, even though the Hessian is Lipschitz, the Hessian of the function can already change by a lot, and the ``pancake'' shape of $\cXs$ is already highly distorted. 
Our proof relies on the following crucial observation: Although we do not know the explicit form of the stuck region, we know it must be very ``thin,'' therefore it cannot have a large volume. The informal statement of the lemma
is as follows:
\begin{lemma}(informal)\label{lem:sketch_informal}
Suppose $\tilde{\x}$ satisfies the precondition of Lemma~\ref{lem:sketch_saddle}, and let $\e_1$ be the smallest eigendirection of $\hess f(\tilde{\x})$. For any $\delta \in (0, 1/3]$ and any two points $\w, \u \in \ball_{\tilde{\x}}(r)$, if $\w-\u = \mu r \e_1$ and $\mu \ge \delta/(2\sqrt{d})$, then at least one of $\w, \u$ is not in the stuck region $\cXs$.
\end{lemma}
% Previous analysis \citep{ge2015escaping} require stepsize and noise radius to have scaling at most $O(\frac{1}{d})$ so that the perturbation ball is so small and $\cXs$ is well-approximated by an extremely ``pancake''.
% To achieve sharp dimension dependence, our step size and perturbation must not polynomially depend on dimension $d$, Hessian Lipschitz condition allow $\cXs$ to become very curvy in this situation. 

% To prove the high probability result in lemma \ref{lem:sketch_saddle}, we only need to show the volumn of $\cXs$ is very small compared to $ \ball_{\tilde{\x}}(r)$, therefore after adding perturbation to $\tilde{\x}$, we only have a very high chance to escape from saddle points efficiently.

% Our general proof strategy is to show $\cXs$ is a very ``thin'' high dimensional band, which only consists of very small proportion of the volumn of the perturbation ball (as shown in Figure \ref{fig:band}). The major challenge here is 

Using this lemma it is not hard to bound the volume of the stuck region: we can draw a straight line along the $\e_1$ direction which intersects the perturbation ball (shown as purple line segment in Figure \ref{fig:flow}). For any two points on this line segment that are at least $\delta r/(2\sqrt{d})$ away from each other (shown as red points $\w, \u$ in Figure \ref{fig:flow}), by Lemma \ref{lem:sketch_informal}, we know at least one of them must not be in $\cXs$. This implies if there is one point $\tilde{\u} \in \cXs$ on this line segment, then $\cXs$ on this line can be at most an interval of length $\delta r/\sqrt{d}$ around $\tilde{\u}$. 
% We can repeat this argument for other straight lines along $\e_1$ direction with different offsets in $\e_2, \e_3, \cdots$ directions. 
This establishes the ``thickness'' of  $\cXs$ 
% can be at most $\delta r/\sqrt{d}$ 
in the $\e_1$ direction, which is turned into an upper bound on
the volume of the stuck region $\cXs$ by standard calculus.

\section{Conclusion}

This paper presents the first (nearly) dimension-free result for 
gradient descent in a general non-convex setting. We present a 
general convergence result and show how it can be further 
strengthened when combined with further structure such as 
strict saddle conditions and/or local regularity/convexity.
% Formally, this paper shows that gradient descent (with perturbations) converges to a second-order stationary point at the same rate as its well-known convergence rate to first-order stationary points up to additional log factors. For the case of strict saddle functions, where convergence to second-order stationary points immediately translates to convergence to local minima, these results guarantee convergence rates to local minima. In the case where the function has strong local structure the paper also shows how to achieve linear convergence. As a simple example, the paper shows that sharp global convergence rates of matrix factorization can be immediately obtained by applying these results as a black-box. 
%  We hope that this result could serve as a first step towards a more general theory with strong, almost dimension free guarantees for non-convex optimization.

There are still many related open problems. First, in the presence of constraints, it is worthwhile to study whether gradient descent still admits similar sharp convergence results.
% though the proof requires adding perturbation once in a while, standard gradient descent with random initialization also seems to work well in practice. Whether similar results hold in that case or whether there are worst-case scenarios where perturbation is necessary is an open question.
Another important question is whether similar techniques can be applied 
to accelerated gradient descent.
We hope that this result could serve as a first step towards a more general theory with strong, almost dimension free guarantees for non-convex optimization.

% Our results serve as general analysis, which can be directly applied to most machine learning applications (including deep learning). 
% Our new results rely on a novel observation of the
% % technic to carefully discribe the 
% geometry and dynamics around saddle points, which might also be interested to general non-convex optimization community.

% \jccomment{Talk about necessarsity of adding noise}

% \jccomment{Talk about accelerated algorithm?}

% Hope our result serve as one step towards general analysis in non-convex optimization.

% \newpage
\bibliographystyle{plainnat}
\bibliography{saddle}

\newpage
\appendix
%!TEX root = main.tex

\section{Detailed Proof of Main Theorem}
\label{app:main}

In this section, we give detailed proof for the main theorem. We will first state two key lemmas that show how the algorithm can make progress when the gradient is large or near a saddle point, and show how the main theorem follows from the two lemmas. Then we will focus on the novel technique in this paper: how to analyze gradient descent near saddle point.

\subsection{General Framework}

In order to prove the main theorem, we need to show that the algorithm will not be stuck at any point that either has a large gradient or is near a saddle point. This idea is similar to previous works (e.g.\citep{ge2015escaping}). We first state a standard Lemma that shows if the current gradient is large, then we make progress in function value.

% \jccomment{Add a new Theorem for general $\theta, \gamma$. And say Theorem 1 as special case.}

\begin{lemma}[Lemma \ref{lem:sketch_gradient} restated]\label{lem:main_grad} 
Assume $f(\cdot)$ satisfies \ref{as:smooth_Lip}, then for gradient descent with stepsize $\eta < \frac{1}{\ell}$, we have:
\begin{equation*}
f(\x_{t+1}) \le f(\x_t) - \frac{\eta}{2}\norm{\grad f(\x_t)}^2
\end{equation*}
\end{lemma}
\begin{proof}
By Assumption \ref{as:smooth_Lip} and its property, we have:
\begin{align*}
f(\x_{t+1}) \le& f(\x_t) + \grad f(\x_t)\trans (\x_{t+1} - \x_t) + \frac{\ell}{2}\norm{\x_{t+1} - \x_t}^2 \\
=& f(\x_t) -\eta \norm{\grad f(\x_t)}^2 + \frac{\eta^2\ell}{2}\norm{\grad f(\x_t)}^2
\le f(\x_t) - \frac{\eta}{2}\norm{\grad f(\x_t)}^2
\end{align*}
\end{proof}

The next lemma says that if we are ``close to a saddle points'', i.e., we are at a point where the gradient is small, but the Hessian has a reasonably large negative eigenvalue. This is the main difficulty in the analysis. We show a perturbation followed by small number of standard gradient descent steps can also make the function value decrease with high probability.
%for a small ball around that point, doing gradient descent from anywhere in that small ball reduces the function value with high probability.

\begin{lemma}[Lemma \ref{lem:sketch_saddle} formal]\label{lem:main_saddle}
There exist absolute constant $c_{\max}$, for $f(\cdot)$ satisfies \ref{as:smooth_Lip}, and any $c\le c_{\max}$, and $\chi\ge1$. Let $\eta, r, g_{\text{thres}}, f_{\text{thres}}, t_{\text{thres}}$ calculated same way as in Algorithm \ref{algo:PGD}. Then, if $\tilde{\x}_t$ satisfies:
\begin{equation*}
\norm{\nabla f(\tilde{\x}_t)} \le g_{\text{thres}} \quad\quad \text{~and~} \quad\quad \lambda_{\min}(\hess f(\tilde{\x}_t)) \le -\sqrt{\rho\epsilon} 
\end{equation*}
Let $\x_t = \tilde{\x}_t + \xi_t$ where $\xi_t$ comes from the uniform distribution over $\ball_0(r)$, 
and let $\x_{t+i}$ be the iterates of gradient descent from $\x_t$ with stepsize 
$\eta$, then with at least probability $1-\frac{d\ell}{\sqrt{\rho\epsilon}}e^{-\chi}$, we have:
\begin{equation*}
f(\x_{t+t_{\text{thres}}}) - f(\tilde{\x}_t) \le -f_{\text{thres}}
\end{equation*}
\end{lemma}

The proof of this lemma is deferred to Section~\ref{sec:mainlemma}. Using this Lemma, we can then prove the main Theorem.

\begingroup
\def\thetheorem{\ref{thm:main}}
\begin{theorem}
There exist absolute constant $c_{\max}$ such that: if $f(\cdot)$ satisfies \ref{as:smooth_Lip},  then for any $\delta>0, \epsilon \le \frac{\ell^2}{\rho}, \Delta_f \ge f(\x_0) - f^\star$, and constant $c \le c_{\max}$, with probability $1-\delta$, the output of $\text{PGD}(\x_0, \ell, \rho, \epsilon, c, \delta, \Delta_f)$ will be $\epsilon-$second order stationary point, and terminate in iterations:
\begin{equation*}
O\left(\frac{\ell(f(\x_0) - f^\star)}{\epsilon^2}\log^{4}\left(\frac{d\ell\Delta_f}{\epsilon^2\delta}\right) \right)
\end{equation*}
\end{theorem}
\addtocounter{theorem}{-1}
\endgroup

%We use the above two lemmas to prove Theorem~\ref{thm:main}.

\begin{proof}%[Proof of Theorem \ref{thm:main}]
Denote $\tilde{c}_{\max}$ to be the absolute constant allowed in Theorem \ref{lem:main_saddle}. 
In this theorem, we let $c_{\max} = \min\{\tilde{c}_{\max}, 1/2\}$, and choose any constant $c \le c_{\max}$.

% Let $c_1, c_2$ be the absolute constant which makes Lemma \ref{lem:main_saddle} hold. W.L.O.G, we can let $c_1<1$. We also choose $\eta = c_1/\ell$. Recall $\Xi = \frac{d \ell\rho(f(\x_0) - f^\star)}{\epsilon \delta}$.

In this proof, we will actually achieve some point satisfying following condition:
\begin{equation} \label{eq:tighter_cond}
 \norm{\grad f(\x)} \le g_{\text{thres}} = \frac{\sqrt{c}}{\chi^2} \cdot \epsilon, \qquad\qquad \lambda_{\min}(\hess f(\x)) \ge - \sqrt{\rho \epsilon}
\end{equation}
Since $c\le1$, $\chi\ge 1$, we have $\frac{\sqrt{c}}{\chi^2} \le 1$, which implies any $\x$ satisfy Eq.\eqref{eq:tighter_cond} is also a \ESSP.

Starting from $\x_0$, we know if $\x_0$ does not satisfy Eq.\eqref{eq:tighter_cond}, there are only two possibilities:
\begin{enumerate}
\item $\norm{\grad f(\x_0)} > g_{\text{thres}}$: In this case, Algorithm \ref{algo:PGD} will not add perturbation. By Lemma \ref{lem:main_grad}:
\begin{equation*}
f(\x_{1}) - f(\x_0) \le  -\frac{\eta}{2} \cdot g_{\text{thres}}^2 = -\frac{c^2}{2\chi^4}\cdot\frac{\epsilon^2}{\ell}
\end{equation*}

\item $\norm{\grad f(\x_0)} \le g_{\text{thres}}$:
In this case, Algorithm \ref{algo:PGD} will add a perturbation of radius $r$, and will perform gradient descent (without perturbations) for the next $t_{\text{thres}}$ steps. Algorithm \ref{algo:PGD} will then check termination condition. If the condition is not met, we must have:
\begin{equation*}
f(\x_{t_{\text{thres}}}) - f(\x_0) \le -f_{\text{thres}} = -\frac{c}{\chi^3}\cdot\sqrt{\frac{\epsilon^3}{\rho}}
\end{equation*}
This means on average every step decreases the function value by
\begin{equation*}
\frac{f(\x_{t_{\text{thres}}}) - f(\x_0)}{t_{\text{thres}}} \le -\frac{c^3}{\chi^4}\cdot\frac{\epsilon^2}{\ell}
\end{equation*}
\end{enumerate}
In case 1, we can repeat this argument for $t = 1$ and in case 2, we can repeat this argument for $t=t_{\text{thres}}$.
Hence, we can conclude as long as algorithm \ref{algo:PGD} has not terminated yet, on average, every step decrease function value by at least $\frac{c^3}{\chi^4}\cdot\frac{\epsilon^2}{\ell}$. However, we clearly can not decrease function value by more than $f(\x_0) - f^\star$, where $f^\star$ is the function value of global minima. This means algorithm \ref{algo:PGD} must terminate within the following number of iterations:
\begin{equation*}
\frac{f(\x_0) - f^\star}{\frac{c^3}{\chi^4}\cdot\frac{\epsilon^2}{\ell}}
= \frac{\chi^4}{c^3}\cdot \frac{\ell(f(\x_0) - f^\star)}{\epsilon^2} = O\left(\frac{\ell(f(\x_0) - f^\star)}{\epsilon^2}\log^{4}\left(\frac{d\ell\Delta_f}{\epsilon^2\delta}\right) \right)
\end{equation*}

~

Finally, we would like to ensure when Algorithm \ref{algo:PGD} terminates, the point it finds is actually an \ESSP. The algorithm can only terminate when the gradient is small, and the function value does not decrease after a perturbation and $t_{\text{thres}}$ iterations. We shall show every time when we add perturbation to iterate $\tilde{\x}_t$, if $\lambda_{\min}(\hess f(\tilde{\x}_t)) < - \sqrt{\rho \epsilon}$, then 
we will have $f(\x_{t+t_{\text{thres}}}) - f(\tilde{\x}_t) \le -f_{\text{thres}}$. Thus, whenever the current point is not an \ESSP, the algorithm cannot terminate.% will not be the output of the algorithm.

According to Algorithm \ref{algo:PGD}, we immediately know $\norm{\nabla f(\tilde{\x}_t)} \le g_{\text{thres}}$ (otherwise we will not add perturbation at time $t$). By Lemma \ref{lem:main_saddle}, we know this event happens with probability at least $1-\frac{d\ell}{\sqrt{\rho\epsilon}}e^{-\chi}$ each time. On the other hand, during one entire run of Algorithm \ref{algo:PGD}, the number of times we add perturbations is at most:
\begin{equation*}
\frac{1}{t_{\text{thres}}} \cdot \frac{\chi^4}{c^3}\cdot \frac{\ell(f(\x_0) - f^\star)}{\epsilon^2}
=\frac{\chi^3}{c}\frac{\sqrt{\rho\epsilon}(f(\x_0) - f^\star)}{\epsilon^2}
\end{equation*}

By union bound, for all these perturbations, with high probability Lemma~\ref{lem:main_saddle} is satisfied. As a result Algorithm \ref{algo:PGD} works correctly. The probability of that is at least
\begin{equation*}
1- \frac{d\ell}{\sqrt{\rho\epsilon}}e^{-\chi} \cdot \frac{\chi^3}{c}\frac{\sqrt{\rho\epsilon}(f(\x_0) - f^\star)}{\epsilon^2}
= 1 -   \frac{\chi^3e^{-\chi}}{c}\cdot  \frac{d\ell(f(\x_0) - f^\star)}{\epsilon^2}
\end{equation*}

Recall our choice of $\chi = 3\max\{\log(\frac{d\ell\Delta_f}{c\epsilon^2\delta}), 4\}$. Since $\chi \ge 12$, we have $\chi^3e^{-\chi} \le e^{-\chi/3}$, this gives:
\begin{equation*}
\frac{\chi^3e^{-\chi}}{c}\cdot  \frac{d\ell(f(\x_0) - f^\star)}{\epsilon^2} 
\le e^{-\chi/3}  \frac{d\ell(f(\x_0) - f^\star)}{c\epsilon^2} \le \delta
\end{equation*}
which finishes the proof.

\end{proof}

\subsection{Main Lemma: Escaping from Saddle Points Quickly}\label{sec:mainlemma}

Now we prove the main lemma (Lemma \ref{lem:main_saddle}), which shows near a saddle point, a small perturbation followed by a small number of gradient descent steps will decrease the function value with high probability. This is the main step where we need new analysis, as the analysis previous works (e.g.\citep{ge2015escaping}) do not work when the step size and perturbation do not depend polynomially in dimension $d$. 

Intuitively, after adding a perturbation, the current point of the algorithm comes from a uniform distribution over a $d$-dimensional ball centered at $\tilde{\x}$, which we call \textbf{perturbation ball}. After a small number of gradient steps, some points in this ball (which we call the {\bf escaping region}) will significantly decrease the function; other points (which we call the {\bf stuck region}) does not see a significant decrease in function value.
% \textcolor{blue}{Rong: The name ``stuck region'' is bad... I can't think of something else for now though.}
% \jccomment{I feel ``stuck'' is quite intuitive, shoud be fine.} 
We hope to show that the escaping region constitutes at least $1-\delta$ fraction of the volume of the perturbation ball. %, so that if we run gradient descent starting from any point in this escaping region, we can always ``efficiently escape saddle point'' (by decreasing sufficient amount of function value in small amount of iterations) . 

However, we do not know the exact form of the function near the saddle point, so the escaping region does not have a clean analytic description. Explicitly computing its volume can be very difficult. Our proof rely on a crucial observation: although we do not know the shape of the stuck region, we know the ``width'' of it must be small, therefore it cannot have a large volume. We will formalize this intuition later in Lemma~\ref{lem:one_in_two}.

The proof of the main lemma requires carefully balancing between different quantities including function value, gradient, parameter space and number of iterations. For clarify, we define following scalar quantities, which serve as the ``units'' for function value, gradient, parameter space, and time (iterations). We will use these notations throughout the proof. 

Let the condition number be the ratio of the smoothness parameter (largest eigenvalue of Hessian) and the negative eigenvalue $\gamma$: $\cn = \ell / \gamma \ge 1$, we define the following units:
\begin{align*}
&\ufun \defeq \eta \ell \frac{\gamma^3}{\rho^2}\cdot \log^{-3}(\logterms),
&&\ugrad \defeq \sqrt{\eta \ell} \frac{\gamma^2}{\rho} \cdot \log^{-2}(\logterms)\\
&\uspace \defeq \sqrt{\eta\ell}\frac{\gamma}{\rho} \cdot \log^{-1}(\logterms), 
&&\utime \defeq \frac{\log (\logterms)}{\eta \gamma}
\end{align*}
Intuitively, if we plug in our choice of learning rate $\eta \ell = O(1)$ (which we will prove later) and hide the logarithmic dependences, we have $\ufun = \tilde{O}(\frac{\gamma^3}{\rho^2}), \ugrad = \tilde{O}(\frac{\gamma^2}{\rho}), \uspace = \tilde{O}(\frac{\gamma}{\rho})$, which is the only way to correctly discribe the units of function value, gradient, parameter space by just $\gamma$ and $\rho$. Moreover, these units are closely related, in particular, we know $\sqrt{\frac{\ufun \cdot \log(\logterms)}{\gamma}} = \frac{\ugrad \cdot \log(\logterms)}{\gamma} = \uspace$.
% \textcolor{blue}{Rong: Maybe add some intuitions on how some of these are chosen.}\jccomment{Is it better now?}

For simplicity of later proofs, we first restate Lemma \ref{lem:main_saddle} into a slightly more general form as follows. Lemma \ref{lem:main_saddle} is directly implied following lemma.
\begin{lemma}[Lemma \ref{lem:main_saddle} restated]\label{lem:main_saddle_restate}
There exists universal constant $c_{\max}$, for $f(\cdot)$ satisfies \ref{as:smooth_Lip}, for any $\delta\in (0, \frac{d\cn}{e}]$, suppose we start with point $\tilde{\x}$ satisfying following conditions:
\begin{equation*}
\norm{\nabla f(\tilde{\x})} \le \ugrad \quad \quad \text{~and~} \quad \quad \lambda_{\min}(\hess f(\tilde{\x})) \le -\gamma
\end{equation*}
Let $\x_0 = \tilde{\x} + \xi$ where $\xi$ come from the uniform distribution over ball with radius $\uspace /(\cn\cdot\log(\logterms))$, 
and let $\x_t$ be the iterates of gradient descent from $\x_0$. Then, when stepsize 
$\eta \le c_{\max} / \ell$, with at least probability $1-\delta$, we have following for any $T \ge \frac{1}{c_{\max}}\utime$:
\begin{equation*}
f(\x_T) - f(\tilde{\x}) \le -\ufun
\end{equation*}
\end{lemma}
Lemma \ref{lem:main_saddle_restate} is almost the same as Lemma \ref{lem:main_saddle}. It is easy to verify that by substituting $\eta = \frac{c}{\ell}, \gamma = \sqrt{\rho \epsilon}$ and $\delta = \frac{d\ell}{\sqrt{\rho\epsilon}}e^{-\chi}$ into Lemma \ref{lem:main_saddle_restate}, we immediately obtain Lemma \ref{lem:main_saddle}.

Now we will formalize the intuition that the ``width'' of stuck region is small.

%To prove Lemma \ref{lem:main_saddle_restate}, we know that after adding perturbation $\xi$, we result in a uniform distribution over a $d$-dimensional ball centered at $\tilde{\x}$ with radius $\uspace /(\cn\cdot\log(\logterms))$, which we call \textbf{perturbation ball}. The key step is to show that there is a \textbf{escaping region} constitute at least $1-\delta$ fraction volumn of this perturbation ball, so that if we run gradient descent starting from any point in this escaping region, we can always ``efficiently escape saddle point'' (by decreasing sufficient amount of function value in small amount of iterations) . 
%
%Note this $d$-dimensional escaping region may not even have analytic discription, thus explicit contructing it can be very difficult. Our proof rely on a key observation as in following lemma.

\begin{lemma}[Lemma \ref{lem:sketch_informal} restated]\label{lem:one_in_two} 
There exists a universal constant $c_{\max}$, for any $\delta\in (0, \frac{d\cn}{e}]$, let $f(\cdot), \tilde{\x}$ satisfies the conditions in Lemma \ref{lem:main_saddle_restate}, and without loss of generality let $\e_1$ be the minimum eigenvector of $\hess f(\tilde{\x})$. Consider two gradient descent sequences $\{\u_t\}, \{\w_t\}$ with initial points $\u_0, \w_0$ satisfying: (denote radius $r = \uspace/(\cn\cdot \log(\logterms))$)
\begin{equation*}
\norm{\u_0 - \tilde{\x}} \le r, \quad \w_0 = \u_0 +\mu \cdot r \cdot \e_1, \quad\mu \in [\delta/(2\sqrt{d}), 1]
\end{equation*}
Then, for any stepsize $\eta \le c_{\max} / \ell$, and any $T \ge \frac{1}{c_{\max}}\utime$, we have:
\begin{equation*}
\min\{f(\u_{T})  - f(\u_0), f(\w_{T})  - f(\w_0)\} \le -2.5\ufun
\end{equation*}
\end{lemma}
Intuitively, lemma \ref{lem:one_in_two} claims for any two points $\u_0, \w_0$ inside the perturbation ball, if $\u_0-\w_0$ lies in the direction of minimum eigenvector of $\hess f(\tilde{\x})$, and $\norm{\u_0-\w_0}$ is greater than threshold $\delta r/ (2\sqrt{d})$, then at least one of two sequences $\{\u_t\}, \{\w_t\}$ will ``efficiently escape saddle point''. In other words, if $\u_0$ is a point in the stuck region, consider any point $\w_0$ that is on a straight line along direction of $\e_1$. As long as $\w_0$ is slightly far ($\delta r/ \sqrt{d}$) from $\u_0$, it must be in the escaping region. This is what we mean by the ``width'' of the stuck region being small. % Thus, we can now lower bound the volumn of this escaping region.
Now we prove the main Lemma using this observation:

\begin{proof}[Proof of Lemma \ref{lem:main_saddle_restate}]
By adding perturbation, in worst case we increase function value by:
\begin{equation*}
f(\x_0) - f(\tilde{\x}) \le \grad f(\tilde{\x})\trans\xi +  \frac{\ell}{2} \norm{\xi}^2 \le 
\ugrad(\frac{\uspace}{\cn \cdot \log(\logterms)}) + \frac{1}{2}\ell(\frac{\uspace}{\cn \cdot \log(\logterms)})^2 \le \frac{3}{2}\ufun
\end{equation*}
On the other hand, let radius $r = \frac{\uspace}{\cn \cdot \log(\logterms)}$. We know $\x_0$ come froms uniform distribution over $\ball_{\tilde{\x}}(r)$. Let $\cXs \subset \ball_{\tilde{\x}}(r)$ denote the set of bad starting points so that if $\x_0 \in \cXs$, then $f(\x_T) - f(\x_0) > -2.5\ufun$ (thus stuck at a saddle point); otherwise if $\x_0 \in B_{\tilde{\x}}(r) - \cXs$, we have $f(\x_T) - f(\x_0) \le -2.5\ufun$.

By applying Lemma \ref{lem:one_in_two}, we know for any $\x_0\in \cXs$, it is guaranteed that $(\x_0 \pm \mu r \e_1) \not \in \cXs $ where $\mu \in [\frac{\delta}{2\sqrt{d}}, 1]$. Denote $I_{\cXs}(\cdot)$ be the indicator function of being inside set $\cXs$; and vector $\x = (x^{(1)}, \x^{(-1)})$, where $x^{(1)}$ is the component along $\e_1$ direction, and $\x^{(-1)}$ is the remaining $d-1$ dimensional vector. Recall $\ball^{(d)}(r)$ be $d$-dimensional ball with radius $r$;  By calculus, this gives an upper bound on the volumn of $\cXs$:
\begin{align*}
\text{Vol}(\cXs) =& \int_{\ball^{(d)}_{\tilde{\x}}(r)}  \mathrm{d}\x \cdot I_{\cXs}(\x)\\
= & \int_{\ball^{(d-1)}_{\tilde{\x}}(r)} \mathrm{d} \x^{(-1)} \int_{\tilde{x}^{(1)} - \sqrt{r^2 - \norm{\tilde{\x}^{(-1)} - \x^{(-1)}}^2}}^{\tilde{x}^{(1)} + \sqrt{r^2 - \norm{\tilde{\x}^{(-1)} - \x^{(-1)}}^2}} \mathrm{d} x^{(1)}  \cdot  I_{\cXs}(\x)\\
% = & \int_{\tilde{x}^{(1)} - r}^{\tilde{x}^{(1)} + r} \mathrm{d} x^{(1)} \int_{B^{(d-1)}_{(x^{(1)}, \tilde{\x}^{(-1)})}(\sqrt{r^2 - (x^{(1)} - \tilde{x}^{(1)})^2})} \mathrm{d} \x^{(-1)}  \cdot  I_{\cXs}(\x)\\
\le & \int_{\ball^{(d-1)}_{\tilde{\x}}(r)} \mathrm{d} \x^{(-1)} \cdot\left(2\cdot \frac{\delta}{2\sqrt{d}}r \right) = \text{Vol}(\ball_0^{(d-1)}(r))\times \frac{\delta r}{\sqrt{d}} 
\end{align*}
% \jccomment{Check if this part of calculation is clear now.}
Then, we immediately have the ratio:
\begin{align*}
\frac{\text{Vol}(\cXs)}{\text{Vol}(\ball^{(d)}_{\tilde{\x}}(r))}
\le \frac{\frac{\delta r}{\sqrt{d}} \times \text{Vol}(\ball^{(d-1)}_0(r))}{\text{Vol} (\ball^{(d)}_0(r))}
= \frac{\delta}{\sqrt{\pi d}}\frac{\Gamma(\frac{d}{2}+1)}{\Gamma(\frac{d}{2}+\frac{1}{2})}
\le \frac{\delta}{\sqrt{\pi d}} \cdot \sqrt{\frac{d}{2}+\frac{1}{2}} \le \delta
\end{align*}
The second last inequality is by the property of Gamma function that $\frac{\Gamma(x+1)}{\Gamma(x+1/2)}<\sqrt{x+\frac{1}{2}}$ as long as $x\ge 0$.
Therefore, with at least probability $1-\delta$, $\x_0 \not \in \cXs$. In this case, we have:
\begin{align*}
f(\x_T) - f(\tilde{\x}) =& f(\x_T)  - f(\x_0) +  f(\x_0) - f(0) \\
\le & -2.5\ufun + 1.5\ufun \le -\ufun
\end{align*}
which finishes the proof.
\end{proof}

\subsection{Bounding the Width of Stuck Region}

In order to prove Lemma~\ref{lem:one_in_two}, we do it in two steps:

\begin{enumerate}
    \item   We first show if gradient descent from $\u_0$ does not decrease function value, then all the iterates must lie within a small ball around $\u_0$ (Lemma~\ref{lem:1st_seq}).
    \item   If gradient descent starting from a point $\u_0$ stuck in a small ball around a saddle point, then gradient descent from $\w_0$ (moving $\u_0$ along $\e_1$ direction for at least a certain distance), will decreases the function value (Lemma~\ref{lem:2nd_seq}).
\end{enumerate}
% Actually, Lemmas~\ref{lem:1st_seq} and~\ref{lem:2nd_seq} mentioned above do not quite argue about function value but only about an approximation to the function value~\eqref{eqn:f-approx}. Lemma~\ref{lem:one_in_two} further refines these statements to talk about function values. We now set about proving these intermediate lemmas. 

Recall we assumed without loss of generality $\e_1$ is the minimum eigenvector of $\hess f(\tilde{\x})$. In this context, we denote $\H \defeq \hess f(\tilde{\x})$, and for simplicity of calculation, we consider following quadratic approximation:
\begin{equation}\label{eqn:f-approx}
\tilde{f}_{\y}(\x) \defeq f(\y) + \grad f(\y)\trans (\x-\y) + \frac{1}{2}(\x-\y)\trans \H (\x-\y)
\end{equation}

Now we are ready to state two lemmas formally:
\begin{lemma}\label{lem:1st_seq}
For any constant $\ca \ge 3$, there exists absolute constant $c_{\max}$: for any $\delta\in (0, \frac{d\cn}{e}]$, let $f(\cdot), \tilde{\x}$ satisfies the condition in Lemma \ref{lem:main_saddle_restate}, for any initial point $\u_0$ with $\norm{\u_0 - \tilde{\x}} \le 2\uspace/(\cn\cdot \log(\logterms))$, define: 
\begin{equation*}
T = \min\left\{ ~\inf_t\left\{t| \tilde{f}_{\u_0}(\u_t) - f(\u_0)  \le -3\ufun \right\},  \ca\utime~\right\}
\end{equation*}
then, for any $\eta \le c_{\max}/\ell$, we have for all $t<T$ that $\norm{\u_t - \tilde{\x}} \le 100( \uspace\cdot \ca )$.
\end{lemma}

\begin{lemma}\label{lem:2nd_seq}
There exists absolute constant $c_{\max}, \ca$ such that: for any $\delta\in (0, \frac{d\cn}{e}]$, let $f(\cdot), \tilde{\x}$ satisfies the condition in Lemma \ref{lem:main_saddle_restate}, and 
sequences $\{\u_t\}, \{\w_t\}$ satisfy the conditions in Lemma \ref{lem:one_in_two}, define:
\begin{equation*}
T = \min\left\{ ~\inf_t\left\{t| \tilde{f}_{\w_0}(\w_t) - f(\w_0)  \le -3\ufun \right\},  \ca\utime~\right\}
\end{equation*}
then, for any $\eta \le c_{\max} / \ell$, if $\norm{\u_t - \tilde{\x}} \le 100 (\uspace\cdot \ca )$ for all $t<T$, we will have $T < \ca \utime$.
\end{lemma}

Note the conclusion $T < \ca \utime$ in Lemma \ref{lem:2nd_seq} equivalently means:
\begin{equation*}
\inf_t\left\{t| \tilde{f}_{\w_0}(\w_t) - f(\w_0)  \le -3\ufun \right\} < \ca\utime
\end{equation*}
That is, for some $T<\ca\utime$, $\{\w_t\}$ sequence ``escape the saddle point'' in the sense of sufficient function value decrease $\tilde{f}_{\w_0}(\w_t) - f(\w_0)  \le -3\ufun$. Now, we are ready to prove Lemma \ref{lem:one_in_two}.

\begin{proof}[Proof of Lemma \ref{lem:one_in_two}]
W.L.O.G, let $\tilde{\x} = 0$ be the origin. Let $(c^{(2)}_{\max}, \ca)$ be the absolute constant so that Lemma \ref{lem:2nd_seq} holds, also let $c^{(1)}_{\max}$ be the absolute constant to make Lemma \ref{lem:1st_seq} holds based on our current choice of $\ca$.
We choose $c_{\max} \le \min\{c^{(1)}_{\max}, c^{(2)}_{\max}\}$ so that our learning rate $\eta \le c_{\max}/\ell$ is small enough which make both Lemma \ref{lem:1st_seq} and Lemma \ref{lem:2nd_seq} hold. Let $T^\star \defeq \ca\utime$ and define:
\begin{equation*}
\Ts = \inf_t\left\{t| \tilde{f}_{\u_0}(\u_t) - f(\u_0)  \le -3\ufun \right\}
\end{equation*}
Let's consider following two cases:

\paragraph{Case $\Ts \le T^\star$:} In this case, by Lemma \ref{lem:1st_seq}, we know $\norm{\u_{\Ts-1}} \le O(\uspace)$, and therefore
\begin{align*}
\norm{\u_{\Ts}} \le &\norm{\u_{\Ts-1}} + \eta \norm{\grad f(\u_{\Ts-1})}
\le \norm{\u_{\Ts-1}} + \eta \norm{\grad f(\tilde{\x})} + \eta \ell \norm{\u_{\Ts-1}} %\\
% \le & (1+\eta\ell) 100(\uspace \cdot \ca) + \eta\ugrad
% \le 300 (\uspace \cdot \ca)
\le O(\uspace)
\end{align*}
By choosing $c_{\max}$ small enough and $\eta \le c_{\max} /\ell$, this gives:
\begin{align*}
f(\u_{\Ts}) - f(\u_0) \le& \grad f(\u_0)\trans (\u_{\Ts}-\u_0) + \frac{1}{2}(\u_{\Ts}-\u_0)\trans \hess f(\u_0) (\u_{\Ts}-\u_0)
+ \frac{\rho}{6} \norm{\u_{\Ts}-\u_0}^3 \\
\le& \tilde{f}_{\u_0}(\u_t) - f(\u_0) + \frac{\rho}{2}\norm{\u_0 - \tilde{\x}}\norm{\u_{\Ts}-\u_0}^2+ \frac{\rho}{6} \norm{\u_{\Ts}-\u_0}^3 \\
\le& -3\ufun + O(\rho \uspace^3) = -3\ufun + O(\sqrt{\eta\ell}\cdot\ufun) \le -2.5\ufun
\end{align*}
By choose $c_{\max} \le \min \{1, \frac{1}{\ca}\}$. We know $\eta < \frac{1}{\ell}$, by Lemma \ref{lem:main_grad}, we know gradient descent always decrease function value. Therefore, for any $T\ge \frac{1}{c_{\max}}\utime \ge \ca\utime = T^\star \ge \Ts$, we have:
\begin{equation*}
f(\u_T) - f(\u_0) \le f(\u_{T^\star})  - f(\u_0) \le f(\u_{\Ts}) - f(\u_0) \le -2.5\ufun
\end{equation*}

\paragraph{Case $\Ts > T^\star$:} In this case, by Lemma \ref{lem:1st_seq}, we know $\norm{\u_t}\le O(\uspace )$ for all $t\le T^\star$. Define 
\begin{equation*}
\Tt = \inf_t\left\{t| \tilde{f}_{\w_0}(\w_t) - f(\w_0)  \le -2\ufun \right\}
\end{equation*}
By Lemma \ref{lem:2nd_seq}, we immediately have $\Tt \le T^\star$. Apply same argument as in first case, we have for all $T\ge \frac{1}{c_{\max}}\utime$ that $f(\w_T) - f(\w_0) \le f(\w_{T^\star})  - f(\w_0) \le -2.5\ufun$.
\end{proof}

Next we finish the proof by proving Lemma~\ref{lem:1st_seq} and Lemma~\ref{lem:2nd_seq}.

\subsubsection{Proof of Lemma \ref{lem:1st_seq}}
In Lemma~\ref{lem:1st_seq}, we hope to show if the function value did not decrease, then all the iterations must be constrained in a small ball. We do that by analyzing the dynamics of the iterations, and we decompose the $d$-dimensional space into two subspaces: a subspace $\mathcal{S}$ which is the span of significantly negative eigenvectors of the Hessian and its orthogonal compliment.

Recall notation $\H \defeq \hess f(\tilde{\x})$ and quadratic approximation $\tilde{f}_{\y}(\x)$ as defined in Eq.\eqref{eqn:f-approx}. Since $\delta \in (0, \frac{d\cn}{e}]$, we always have $\log(\logterms)\ge 1$.
W.L.O.G, set $\u_0 = 0$ to be the origin, by gradient descent update function, we have:
\begin{align}
\u_{t+1} =& \u_t - \eta \grad f(\u_t) \nn\\
=& \u_t - \eta \grad f(0) - \eta \left[\int_{0}^1 \hess f(\theta\u_t) \mathrm{d}\theta\right] \u_t \nn \\
= &\u_t - \eta \grad f(0) - \eta(\H + \Delta_t) \u_t\nn\\
= &(\I - \eta \H - \eta \Delta_t) \u_t - \eta \grad f(0) \label{eq:update_u}
\end{align}
Here, $\Delta_t \defeq \int_{0}^1 \hess f(\theta\u_t) \mathrm{d}\theta - \H$. By Hessian Lipschitz, we have $\norm{\Delta_t} \le \rho(\norm{\u_t} + \norm{\tilde{\x}})$, and by smoothness of the gradient, we have 
$\norm{\grad f(0)} \le \norm{\grad f(\tilde{\x})} + \ell\norm{\tilde{\x}} \le \ugrad + \ell \cdot 2\uspace/(\cn\cdot \log(\logterms)) \le 3\ugrad$.
% Here, $\norm{\Delta_t} \le \norm{\u_t} + \norm{\tilde{\x}}$, and 
% $\norm{\grad f(0)} \le \norm{\grad f(\tilde{\x})} + \ell\norm{\tilde{\x}} \le \ugrad + \ell \cdot 2\uspace /(\cn\cdot \log(\logterms)) \le 3\ugrad$.

We will now compute the projections of $\u_t$ in different eigenspaces of $\H$. Let $\S$ be the subspace spanned by all eigenvectors of $\H$ whose eigenvalue is less than $-\frac{\gamma}{\overh}$. $\S^c$ denotes the subspace of remaining eigenvectors. Let $\balpha_t$ and $\bbeta_t$ denote the projections of $\u_t$ onto $\S$ and $\S^c$ respectively i.e., $\balpha_t = \proj_{\S} \u_{t}$, and $\bbeta_t = \proj_{\Scomp} \u_{t}$.
We can decompose the update equations Eq.\eqref{eq:update_u} into:
% Suppose we expand the coordinate system align with eigenvectors of $\H$, let $\S$ be the subspace consists of all eigenvectors whose eigenvalue less than $-\frac{\gamma}{\overh}$. $\S^c$ be the subspace of remaining eigenvectors.
% Denote $\balpha_t = \proj_{\S} \u_{t}$, and $\bbeta_t = \proj_{\Scomp} \u_{t}$, where $\proj_{\S}(\cdot)$ means projection to subspace of $\S$.
% We can decompose the update equation Eq.\eqref{eq:update_u} into:
\begin{align}
\balpha_{t+1} =& (\I-\eta\H)\balpha_{t} - \eta\proj_\S \Delta_t \u_t - \eta\proj_\S \grad f(0) \label{eq:update_S}\\
\bbeta_{t+1} =& (\I-\eta\H)\bbeta_{t} - \eta\proj_{\S^c} \Delta_t \u_t - \eta\proj_{\S^c} \grad f(0)
\label{eq:update_Sc}
\end{align}

By definition of $T$, we know for all $t<T$:
\begin{align*}
-3\ufun < \tilde{f}_{0}(\u_t)  - f(0) =& \grad f(0)\trans \u_t - \frac{1}{2}\u_t\trans \H \u_t
\le \grad f(0)\trans \u_t - \frac{\gamma}{2}\frac{\norm{\balpha_t}^2}{\overh} + \frac{1}{2}\bbeta_t\trans\H\bbeta_t
\end{align*}
Combined with the fact $\norm{\u_t}^2 = \norm{\balpha_t}^2 + \norm{\bbeta_t}^2$, we have:
\begin{align*}
\norm{\u_t}^2 \le & \frac{2\overh}{\gamma} \left(3\ufun + \grad f(0)\trans \u_t + \frac{1}{2}\bbeta_t\trans\H\bbeta_t\right) + \norm{\bbeta_t}^2 \\
\le & 14 \cdot \max\left\{  \frac{\ugrad\overh}{\gamma}\norm{\u_t}, ~\frac{\ufun \overh}{\gamma}, ~\frac{\bbeta_t\trans\H\bbeta_t \overh}{\gamma}, ~\norm{\bbeta_t}^2  \right\}
\end{align*}
where last inequality is due to $\norm{\grad f(0)} \le 3 \ugrad$. This gives:
\begin{equation}\label{eq:upper_bound_space}
\norm{\u_t}  \le 14 \cdot \max\left\{ \frac{\ugrad\overh}{\gamma}, ~ \sqrt{\frac{\ufun \overh}{\gamma}}, ~\sqrt{\frac{\bbeta_t\trans\H\bbeta_t \overh}{\gamma}}, ~\norm{\bbeta_t} \right\}
\end{equation}
% To prove the lemma, we only need to show last two terms are bounded.

Now, we use induction to prove that
\begin{equation}\label{eqn:ut-bound}
\norm{\u_t} \le 100(\uspace \cdot \ca)
\end{equation}
% Clearly, initial condition holds. Suppose for all $\tau\le t$, induction holds. Then consider for $t+1 < T$, by Eq.\eqref{eq:upper_bound_space}, we only need to bound last two terms $\norm{\bbeta_{t+1}}$ and $\bbeta_{t+1}\trans \H \bbeta_{t+1}$.
Clearly Eq.\eqref{eqn:ut-bound} is true for $t=0$ since $\u_0 = 0$. Suppose Eq.\eqref{eqn:ut-bound} is true for all $\tau\le t$. We will now show that Eq.\eqref{eqn:ut-bound} holds for $t+1 < T$. Note that by the definition of $\uspace$, $\ufun$ and $\ugrad$, we only need to bound the last two terms of Eq.\eqref{eq:upper_bound_space} i.e., $\norm{\bbeta_{t+1}}$ and $\bbeta_{t+1}\trans \H \bbeta_{t+1}$.

By update function of $\bbeta_t$ (Eq.\eqref{eq:update_Sc}), we have:
\begin{align} \label{eq:update_u_simplified}
\bbeta_{t+1} \le& (\I -\eta\H)\bbeta_t + \eta \bdelta_t
\end{align}
and the norm of $\bdelta_t$ is bounded as follows:
\begin{align}
\norm{\bdelta_t} &\le \norm{\Delta_t}\norm{\u_t} + \norm{\grad f(0)} \nn\\
& \le \rho \left(\norm{\u_t} + \norm{\tilde{\x}}\right) \norm{\u_t} + \norm{\grad f(0)} \nn\\
&\le \rho \cdot 100\ca(100\ca+ 2/(\cn\cdot \log(\logterms)))\uspace^2  + \ugrad \nn\\
&= [100\ca(100\ca+2)\sqrt{\eta\ell} + 1]\ugrad \le 2\ugrad \label{eqn:err-H-bound}
\end{align}
The last step follows by choosing small enough constant $c_{\max} \le \frac{1}{100\ca(100\ca+2)}$ and stepsize $\eta < c_{\max}/\ell$.

% ~

% where $\norm{\bdelta_t} \le \norm{\Delta_t}\norm{\u_t} + \norm{\grad f(0)}
% \le O(\rho(\uspace \cdot \ca)^2) + \ugrad \le 2\ugrad$.
% The last step is achieved by the requirement for step size $\eta < O(1/\ell)$.

\paragraph{Bounding $\norm{\bbeta_{t+1}}$:}
Combining Eq.\eqref{eq:update_u_simplified}, Eq.\eqref{eqn:err-H-bound} and using the definiton of $\Scomp$, we have:
\begin{equation*}
\norm{\bbeta_{t+1}} \le (1+ \frac{\eta \gamma}{\overh}) \norm{\bbeta_t} + 2\eta\ugrad
\end{equation*}
Since $\norm{\bbeta_0} = 0$ and $t+1 \le T$, by applying above relation recurrsively, we have:
\begin{equation}
\norm{\bbeta_{t+1}} \le \sum_{\tau = 0}^{t}2(1+ \frac{\eta \gamma}{\overh})^\tau\eta\ugrad \le 2\cdot 3\cdot T\eta \ugrad
\le 6(\uspace \cdot \ca) \label{eq:bound_u1}
\end{equation}
The second last inequality is because $T \le \ca \utime$ by definition, so that $(1+ \frac{\eta \gamma}{\overh})^T \le 3$.

\paragraph{Bounding $\bbeta_{t+1}\trans\H\bbeta_{t+1}$:} Using Eq.\eqref{eq:update_u_simplified}, we can also write the update equation as:
\begin{equation*}
\bbeta_t = \sum_{\tau = 0}^{t-1} (\I - \eta\H)^{\tau} \bdelta_{\tau}
\end{equation*}
Combining with Eq.\eqref{eqn:err-H-bound}, this gives
\begin{align*}
\bbeta_{t+1}\trans \H \bbeta_{t+1} =& \eta^2\sum_{\tau_1 = 0}^t \sum_{\tau_2 = 0}^t 
\bdelta_{\tau_1}\trans (\I - \eta \H)^{\tau_1}\H(\I - \eta \H)^{\tau_2}\bdelta_{\tau_2} \\
\le &\eta^2\sum_{\tau_1 = 0}^t \sum_{\tau_2 = 0}^t \norm{\bdelta_{\tau_1}} 
\norm{(\I - \eta \H)^{\tau_1}\H(\I - \eta \H)^{\tau_2}}\norm{\bdelta_{\tau_2}} \\
\le& 4\eta^2 \ugrad^2\sum_{\tau_1 = 0}^t \sum_{\tau_2 = 0}^t  \norm{(\I - \eta \H)^{\tau_1}\H(\I - \eta \H)^{\tau_2}}
\end{align*}
Let the eigenvalues of $\H$ to be $\{\lambda_i\}$, then for any $\tau_1, \tau_2 \ge 0$, we know the eigenvalues of 
$(\I - \eta \H)^{\tau_1}\H(\I - \eta \H)^{\tau_2}$ are $\{\lambda_i(1-\eta \lambda_i)^{\tau_1 + \tau_2}\}$.
Let $g_t(\lambda) \defeq \lambda (1-\eta \lambda)^t$, and setting its derivative to zero, we obtain:
\begin{equation*}
\grad g_t(\lambda) = (1-\eta\lambda)^t -t\eta\lambda(1-\eta\lambda)^{t-1} = 0
\end{equation*}
We see that $\lambda_t^\star = \frac{1}{(1+t)\eta}$ is the unique maximizer, and $g_t(\lambda)$ is monotonically increasing in $(-\infty, \lambda_t^\star]$. This gives:
\begin{equation*}
\norm{(\I - \eta \H)^{\tau_1}\H(\I - \eta \H)^{\tau_2}}
 = \max_i \lambda_i(1-\eta \lambda_i)^{\tau_1 + \tau_2}
 \le \hat{\lambda}(1-\eta\hat{\lambda})^{\tau_1 + \tau_2} \le \frac{1}{(1+\tau_1+\tau_2)\eta}
\end{equation*}
where $\hat{\lambda} = \min\{\ell, \lambda_{\tau_1 + \tau_2}^\star\}$. Therefore, we have:
\begin{align}
\bbeta_{t+1}\trans \H \bbeta_{t+1} 
\le& 4\eta^2 \ugrad^2\sum_{\tau_1 = 0}^t \sum_{\tau_2 = 0}^t  \norm{(\I - \eta \H)^{\tau_1}\H(\I - \eta \H)^{\tau_2}}  \nn\\
\le& 4\eta \ugrad^2\sum_{\tau_1 = 0}^t \sum_{\tau_2 = 0}^t \frac{1}{1+\tau_1+\tau_2}
\le 8\eta T\ugrad^2
\le 8 \uspace^2 \gamma \ca \cdot \log^{-1}(\logterms) \label{eq:bound_u2}
\end{align}
The second last inequality is because by rearrange summation:
\begin{equation*}
\sum_{\tau_1 = 0}^t \sum_{\tau_2 = 0}^t \frac{1}{1+\tau_1+\tau_2}
= \sum_{\tau = 0}^{2t} \min\{1+\tau, 2t+1-\tau\} \cdot \frac{1}{1+\tau} \le 2t+1 < 2T
\end{equation*}

~

Finally, substitue Eq.\eqref{eq:bound_u1} and Eq.\eqref{eq:bound_u2} into Eq.\eqref{eq:upper_bound_space}, this gives:
\begin{align*}
\norm{\u_{t+1}}  \le& 14 \cdot \max\left\{ \frac{\ugrad\overh}{\gamma}, ~ \sqrt{\frac{\ufun \overh}{\gamma}}, ~\sqrt{\frac{\bbeta_t\trans\H\bbeta_t \overh}{\gamma}}, ~\norm{\bbeta_t} \right\} \\
\le & 100 (\uspace \cdot \ca)
\end{align*}
This finishes the induction as well as the proof of the lemma. \hfill $\qed$

\subsubsection{Proof of Lemma \ref{lem:2nd_seq}}
% In high level, we want to say for close points of form $\w_0 = \u_0 + \v_0$, where $\v_0 = (c_0\uspace/\cn) \cdot \e_1$ where $\e_1$ is the minimum eigenvector direction of $\H$, and $c_0 \in [0.01, 1]$.
% If sequence started at $\u_0$ stuck around saddle point, then the squence started at $\w_0$ will escape saddle point.

In this Lemma we try to show if all the iterates from $\u_0$ are constrained in a small ball, iterates from $\w_0$ must be able to decrease the function value. In order to do that, we keep track of vector $\v$ which is the difference between $\u$ and $\w$. Similar as before, we also decompose $\v$ into different eigenspaces. However, this time we only care about the projection of $\v$ on the direction $\e_1$ and its orthognal subspace.

Again, recall notation $\H \defeq \hess f(\tilde{\x})$, $\e_1$ as minimum eigenvector of $\H$ and quadratic approximation $\tilde{f}_{\y}(\x)$ as defined in Eq.\eqref{eqn:f-approx}. Since $\delta \in (0, \frac{d\cn}{e}]$, we always have $\log(\logterms)\ge 1$.
W.L.O.G, set $\u_0 = 0$ to be the origin. Define $\v_t = \w_t - \u_t$, by assumptions in Lemma \ref{lem:2nd_seq}, we have $\v_0 = \mu [\uspace/(\cn\cdot \log(\logterms))] \e_1, ~\mu \in [\delta/(2\sqrt{d}), 1]$. Now, consider the update equation for $\w_t$:
\begin{align*}
\u_{t+1} + \v_{t+1} = \w_{t+1} = &\w_t - \eta \grad f(\w_t) \nn\\
=&\u_t + \v_t - \eta \grad f(\u_t + \v_t) \nn\\
=&\u_t + \v_t - \eta \grad f(\u_t) - \eta \left[\int_{0}^1 \hess f(\u_t + \theta\v_t) \mathrm{d}\theta\right] \v_t \nn\\ 
= &\u_t + \v_t - \eta \grad f(\u_t) - \eta(\H + \Delta'_t) \v_t\nn\\
= &\u_t - \eta \grad f(\u_t) + (\I - \eta \H - \eta \Delta'_t) \v_t
\end{align*}
where $\Delta'_t \defeq \int_{0}^1 \hess f(\u_t + \theta\v_t) \mathrm{d}\theta - \H$. By Hessian Lipschitz, we have $\norm{\Delta'_t} \le \rho( \norm{\u_t} + \norm{\v_t}+ \norm{\tilde{\x}})$. 
This gives the dynamic for $\v_t$ satisfy:
\begin{equation}\label{eq:v_dynamic}
\v_{t+1} = (\I - \eta \H - \eta \Delta'_t) \v_t
\end{equation}

Since $\norm{\w_0 -\tilde{\x}} \le \norm{\u_0 - \tilde{\x}} + \norm{\v_0} \le \uspace/(\cn\cdot \log(\logterms))$, directly applying Lemma \ref{lem:1st_seq}, we know $\w_t\le 100( \uspace\cdot\ca )$ for all $t \le T$. By condition of Lemma \ref{lem:2nd_seq}, we know $\norm{\u_t} \le 100( \uspace\cdot\ca )$ for all $t<T$.
This gives:
\begin{equation} \label{eq:bound_v}
\norm{\v_t} \le \norm{\u_t} + \norm{\w_t} \le 200( \uspace\cdot \ca) \text{~for all~} t<T
\end{equation}
This in sum gives for $t<T$:
\begin{equation*}
\norm{\Delta'_t} \le \rho( \norm{\u_t} + \norm{\v_t}+ \norm{\tilde{\x}})
\le \rho( 300\ca \uspace+ \uspace/(\cn\cdot \log(\logterms)))
\le \rho\uspace (300 \ca + 1)
\end{equation*}

On the other hand, denote $\psi_t$ be the norm of $\v_t$ projected onto $\e_1$ direction, and $\varphi_t$ be the norm of $\v_t$ projected onto remaining subspace. Eq.\eqref{eq:v_dynamic} gives us:
\begin{align*}
\psi_{t+1} \ge& (1+\gamma \eta)\psi_t - \mu\sqrt{\psi_t^2 + \varphi_t^2}\\
\varphi_{t+1} \le &(1+\gamma\eta)\varphi_t + \mu\sqrt{\psi_t^2 + \varphi_t^2}
\end{align*}
where $\mu = \eta\rho  \uspace(300 \ca + 1)$. We will now prove via induction that for all $t < T$:
\begin{equation}
\varphi_t \le 4 \mu t \cdot \psi_t\label{eqn:phit}
\end{equation}
By hypothesis of Lemma \ref{lem:2nd_seq}, we know $\varphi_0 = 0$, thus the base case of induction holds. Assume Eq.\eqref{eqn:phit} is true for $\tau\le t$, For $t+1 \le T$, we have:
\begin{align*}
4\mu(t+1)\psi_{t+1} 
\ge & 4\mu (t+1) \left( (1+\gamma \eta)\psi_t - \mu \sqrt{\psi_t^2 + \varphi_t^2}\right) \\
\varphi_{t+1} \le &4 \mu  t(1+\gamma\eta) \psi_t + \mu \sqrt{\psi_t^2 + \varphi_t^2}
\end{align*} 
From above inequalities, we see that we only need to show:
% To prove the induction assumption, we only need to show:
\begin{equation*}
 \left(1+4\mu (t+1)\right)\sqrt{\psi_t^2 + \varphi_t^2}
 \le 4 (1+\gamma \eta)\psi_t
\end{equation*}
By choosing $\sqrt{c_{\max}}\le \frac{1}{300\ca+1}\min\{\frac{1}{2\sqrt{2}}, \frac{1}{4\ca}\}$, and $\eta \le c_{\max}/\ell$, we have 
\begin{equation*}
4\mu (t+1) \le 4\mu T \le 
4\eta\rho  \uspace(300 \ca + 1)\ca\utime =4\sqrt{\eta\ell}(300 \ca + 1)\ca\le 1
\end{equation*}
This gives:
\begin{align*}
4 (1+\gamma \eta)\psi_t \ge 4\psi_t \le 2\sqrt{2\psi_t^2}\ge \left(1+4\mu (t+1)\right)\sqrt{\psi_t^2 + \varphi_t^2}
\end{align*}
which finishes the induction.

~

Now, we know $\varphi_t \le 4  \mu t \cdot \psi_t \le \psi_t$, this gives:
\begin{equation}
\psi_{t+1} \ge (1+\gamma \eta)\psi_t - \sqrt{2}\mu\psi_t
\ge (1+\frac{\gamma \eta}{2})\psi_t \label{eq:growth_v}
\end{equation}
where the last step follows from $\mu = \eta \rho \uspace(300 \ca + 1) \le  \sqrt{c_{\max}}(300 \ca + 1) \gamma \eta \cdot\log^{-1}(\logterms) < \frac{\gamma \eta}{2\sqrt{2}}$.

Finally, combining Eq.\eqref{eq:bound_v} and \eqref{eq:growth_v} we have for all $t<T$:
\begin{align*}
200( \uspace\cdot \ca)
\ge &\norm{\v_t} \ge \psi_t \ge (1+\frac{\gamma \eta}{2})^t \psi_0\\
= &(1+\frac{\gamma \eta}{2})^t c_0 \frac{\uspace}{\cn}\log^{-1}(\logterms)
\ge (1+\frac{\gamma \eta}{2})^t \frac{\delta}{2\sqrt{d}}\frac{\uspace}{\cn}\log^{-1}(\logterms)
\end{align*}
This implies:
\begin{equation*}
T < \frac{1}{2}\frac{\log (400 \frac{\cn\sqrt{d}}{\delta}\cdot \overh)}{\log (1+\frac{\gamma \eta}{2})}
\le \frac{\log (400 \frac{\cn\sqrt{d}}{\delta}\cdot \overh)}{\gamma\eta}
\le (2 + \log (400 \ca))\utime
\end{equation*}
The last inequality is due to $\delta\in (0, \frac{d\cn}{e}]$ we have $\log (\logterms) \ge 1$.
% \jccomment{Make this point more clear} 
By choosing constant $\ca$ to be large enough to satisfy $2 + \log (400 \ca) \le \ca$, we will have
$T < \ca\utime $, which finishes the proof. \hfill $\qed$

%!TEX root = main.tex

\section{Improve Convergence by Local Structure}
In this section, we show if the objective function has nice {\em local structure} (e.g. satisfies Assumptions~\ref{as:sc} or \ref{as:regularity}), then it is possible to combine our analysis with the local analysis in order to get very fast convergence to a local minimum.

In particular, we prove Theorem \ref{thm:main_local}. 

% \textcolor{blue}{Rong: Again, we should restate the Theorem here.}
\begingroup
\def\thetheorem{\ref{thm:main_local}}
\begin{theorem}
There exist absolute constant $c_{\max}$ such that: if $f(\cdot)$ satisfies \ref{as:smooth_Lip}, \ref{as:strict_saddle}, and \ref{as:sc} (or \ref{as:regularity}), then for any $\delta>0, \epsilon>0, \Delta_f \ge f(\x_0) -f^\star$, constant $c \le c_{\max}$, 
let $\tilde{\epsilon} = \min(\theta, \gamma^2/\rho)$, with probability $1-\delta$, the output of $\text{PGDli}(\x_0, \ell, \rho, \tilde{\epsilon}, c, \delta, \Delta_f, \beta)$ will be $\epsilon$-close to $\cXstar$ in iterations:
\begin{equation*}
O\left(\frac{\ell(f(\x_0) - f^\star)}{\tilde{\epsilon}^2}\log^{4}\left(\frac{d\ell\Delta_f}{\tilde{\epsilon}^2\delta}\right)  + \frac{\beta}{\alpha}\log \frac{\zeta}{\epsilon}\right)
\end{equation*}
\end{theorem}
\addtocounter{theorem}{-1}
\endgroup

\begin{proof} %[Proof of Theorem \ref{thm:main_local}]
Theorem \ref{thm:main_local} runs $\text{PGDli}(\x_0, \ell, \rho, \tilde{\epsilon}, c, \delta, \Delta_f, \beta)$. According to algorithm \ref{algo:PGDli}, we know it calls $\text{PGD}(\x_0, \ell, \rho, \epsilon, c, \delta, \Delta_f)$ first (denote its output as $\hat{\x}$), then run standard gradient descent with learning rate $\frac{1}{\beta}$ starting from $\hat{\x}$. 

By Corollary \ref{cor:main_strictsaddle}, we know $\hat{\x}$ is already in the $\zeta$-neighborhood of $\cXstar$, where $\cXstar$ is the set of local minima. Therefore, to prove this theorem, we only need to show prove following two claims:
\begin{enumerate}
\item Suppose $\{\x_t\}$ is the sequence of gradient descent starting from $\x_0 = \hat{\x}$ with learning rate $\frac{1}{\beta}$, then $\x_t$ is always in the $\zeta$-neighborhood of $\cXstar$.
\item Local structure (assumption \ref{as:sc} or \ref{as:regularity}) allows iterates to converge to points $\epsilon$-close to $\cXstar$ within $O(\frac{\beta}{\alpha}\log \frac{\zeta}{\epsilon})$ iterations.
\end{enumerate}

We will focus on Assumption \ref{as:regularity} (as we will later see Assumption~\ref{as:sc} is a special case of Assumption~\ref{as:regularity}). Assume $\x_t$ is in $\zeta$-neighborhood of $\cXstar$, by gradient updates and the definition of projection, we have:
\begin{align*}
\norm{\x_{t+1} - \projX(\x_{t+1}) }^2
\le& \norm{\x_{t+1} - \projX(\x_{t}) }^2
= \norm{\x_{t} - \eta\grad f(\x_t) - \projX(\x_{t}) }^2 \\
=& \norm{\x_{t} - \projX(\x_{t}) }^2
- 2\eta \la \grad f(\x_t), \x_{t} - \projX(\x_{t})\ra
+ \eta^2 \norm{\grad f(\x_t)}^2 \\
\le & \norm{\x_{t} - \projX(\x_{t}) }^2 - \eta\alpha\norm{\x_t- \projX(\x_t)}^2 + (\eta^2 - \frac{\eta}{\beta}) \norm{\grad f(\x)}^2 \\
\le & (1-\frac{\alpha}{\beta})\norm{\x_{t} - \projX(\x_{t}) }^2
\end{align*}
The second last inequality is due to $(\alpha, \beta)$-regularity condition.
The last inequality is because of the choice $\eta = \frac{1}{\beta}$.

There are two consequences of this calculation. First, it shows $\norm{\x_{t+1} - \projX(\x_{t+1}) }^2 \le \norm{\x_{t} - \projX(\x_{t}) }^2$. As a result if $\x_t$ in $\zeta$-neighborhood of $\cXstar$, $\x_{t+1}$ is also in this $\zeta$-neighborhood. Since $\x_0$ is in the $\zeta$-neighborhood by Corollary \ref{cor:main_strictsaddle}, by induction we know all later iterations are in the same neighborhood.

Now, since we know all the points $\x_t$ are in the neighborhood, the equation also shows linear convergence rate $(1-\frac{\alpha}{\beta})$. The initial distance is bounded by $\norm{\x_{0} - \projX(\x_{0}) } \le \zeta$, therefore to converge to points $\epsilon$-close to $\cXstar$, we only need the following number of iterations:
\begin{equation*}
\frac{\log (\epsilon/\zeta)}{\log (1-\alpha/\beta)} = O(\frac{\beta}{\alpha} \log \frac{\zeta}{\epsilon}).
\end{equation*}
This finishes the proof under Assumption~\ref{as:regularity}.

Finally, we argue assumption \ref{as:sc} implies \ref{as:regularity}. First, notice that if a function is locally strongly convex, then its local minima are isolated: for any two points $\x,\x'\in \cXstar$, the local region $B_\x(\zeta)$ and $B_{\x'}(\zeta)$ must be disjoint (otherwise function $f(\x)$ is strongly convex in connected domain $B_\x(\zeta) \cup B_{\x'}(\zeta)$ but has two distinct local minima, which is impossible). Therefore, W.L.O.G, it suffices to consider one perticular disjoint region, with unique local minimum we denote as $\x^\star$, clearly, for all $\x \in B_{\x^\star}(\zeta)$ we have $\projX(\x) = \x^\star$.

Now by $\alpha$-strong convexity:
\begin{equation}\label{eq:sc_local}
f(\x^\star) \ge f(\x) + \la \grad f(\x), \x^\star -\x\ra + \frac{\alpha}{2}\norm{\x-\x^\star}^2
\end{equation}
On the other hand, for any $\x$ in this $\zeta$-neighborhood, we already proved $\x - \frac{1}{\beta} \grad f(\x)$ also in this $\zeta$-neighborhood. By $\beta$-smoothness, we also have:
\begin{equation}\label{eq:smooth_local}
f(\x - \frac{1}{\beta} \grad f(\x)) \le  f(\x) - \frac{1}{2\beta} \norm{\grad f(\x)}^2
\end{equation} 
Combining Eq.\eqref{eq:sc_local} and Eq.\eqref{eq:smooth_local}, and using the fact $f(\x^\star) \le f(\x - \frac{1}{\beta} \grad f(\x))$, we get:
\begin{equation*}
\la \grad f(\x), \x - \x^\star \ra  \ge \frac{\alpha}{2}\norm{\x-\x^\star}^2 + \frac{1}{2\beta} \norm{\grad f(\x)}^2
\end{equation*}
which finishes the proof.
\end{proof}
%!TEX root = main.tex

\section{Geometric Structures of Matrix Factorization Problem}%Proof of Theorem \ref{thm:mf_global}}

In this Section we investigate the global geometric structures of the matrix factorization problem. These properties are summarized in Lemmas~\ref{lem:mf_smooth} and \ref{lem:mf_strictsaddle}. Such structures allow us to apply our main Theorem and get fast convergence (as shown in Theorem~\ref{thm:mf_global}).

Note that our main results Theorems~\ref{thm:main} and~\ref{thm:main_local} are proved for functions $f(\cdot)$ whose input $\x$ is a vector. For the current function in~\ref{eq:obj}, though the input $\U \in \R^{d\times r}$  is a matrix, we can always vectorize it to be a vector in $\R^{dr}$ and apply our results. However, for simplicity of presentation, we still write everything in matrix form (without explicit vectorization), while the reader should keep in mind the operations are same if one vectorizes everything first. 

Recall for vectors we use $\norm{\cdot}$ to denote the 2-norm, and for matrices we use $\norm{\cdot}$ and $\fnorm{\cdot}$ to denote spectral norm, and Frobenius norm respectively. Furthermore, we always use $\sigma_i(\cdot)$ to denote the $i$-th largest singular value of the matrix. 
% \praneeth{Also add some notation either here or in the main paper on notations for matrices. Operator norm, Frobenius norm etc.} \jccomment{Remember to mention it in main paper} 

We first show how the geometric properties (Lemma \ref{lem:mf_smooth} and Lemma \ref{lem:mf_strictsaddle}) imply a fast convergence (Theorem \ref{thm:mf_global}).

% \textcolor{blue}{Rong: Again we should restate Theorem~\ref{thm:mf_global}.}
\begingroup
\def\thetheorem{\ref{thm:mf_global}}
\begin{theorem}
There exists an absolute constant $c_{\max}$ such that the following holds. For matrix factorization~\eqref{eq:obj}, for any $\delta >0$ and constant $c\le c_{\max}$, let $\Gamma^{1/2} \defeq 2\max\{\norm{\U_0}, 3(\sigstarl)^{1/2}\}$,
suppose we run $\text{PGDli}(\U_0, 8\Gamma, 12\Gamma^{1/2}, \frac{(\sigstarr)^{2}}{108\Gamma^{1/2}}, c, \delta, \frac{r\Gamma^2}{2}, 10\sigstarl)$, then:
\begin{enumerate}
\item With probability 1, the iterates satisfy $\norm{\U_t} \le \Gamma^{1/2}$ for every $t\ge 0$.
\item With probability $1-\delta$, the output will be $\epsilon$-close to global minima set $\cXstar$ in following iterations:
\begin{equation*}
O\left(r\left(\frac{ \Gamma}{\sigstarr}\right)^4\log^{4}\left(\frac{d \Gamma}{\delta\sigstarr}\right)  + \frac{\sigstarl}{\sigstarr}\log \frac{\sigstarr}{\epsilon} \right)
\end{equation*}
\end{enumerate}
\end{theorem}
\addtocounter{theorem}{-1}
\endgroup

\begin{proof}[Proof of Theorem \ref{thm:mf_global}]
Denote $\tilde{c}_{\max}$ to be the absolute constant allowed in Theorem \ref{thm:main_local}. In this theorem, we let $c_{\max} = \min\{\tilde{c}_{\max}, \frac{1}{2}\}$, and choose any constant $c \le c_{\max}$.

Theorem~\ref{thm:mf_global} consists of two parts. In part 1 we show that the iterations never bring the matrix to a very large norm, while in part 2 we apply our main Theorem to get fast convergence. We will first prove the bound on number of iterations assuming the bound on the norm. We will then proceed to prove part $1$.

~

\noindent \textbf{Part 2:} Assume part $1$ of the theorem is true i.e., with probability $1$, the iterates satisfy $\norm{\U_t} \le \Gamma^{1/2}$ for every $t\ge 0$. In this case, although we are doing unconstrained optimization, we can still use the geometric properties that hold inside this region.
%, because algorithm will never leave this region $\{\U | \norm{\U}^2 \le \Gamma\}$
. 

By Lemma~\ref{lem:mf_smooth} and~\ref{lem:mf_strictsaddle}, we know objective function Eq.\eqref{eq:obj} is 
$8\Gamma$-smooth, $12\Gamma^{1/2}$-Lipschitz Hessian, $(\frac{1}{24}(\sigstarr)^{3/2}, \frac{1}{3}\sigstarr, \frac{1}{3}(\sigstarr)^{1/2})$-strict saddle, and holds $(\frac{2}{3}\sigstarr, 10\sigstarl)$-regularity condition in $\frac{1}{3}(\sigstarr)^{1/2}$ neighborhood of local minima (also global minima) $\cXstar$. Furthermore, note $f^\star = 0$ and recall $\Gamma^{1/2} = 2\max\{\norm{\U_0}, 3(\sigstarl)^{1/2}\}$, then, we have:
\begin{align*}
f(\U_0) - f^\star = \fnorm{\U_0\U_0\trans - \M^\star}^2 
\le 2r\norm{\U_0\U_0\trans - \M^\star}^2\le \frac{r \Gamma^2}{2}.
\end{align*}
Thus, we can choose $\Delta_f = \frac{r \Gamma^2}{2}$. Substituting the corresponding parameters from Theorem~\ref{thm:main_local}, we know by running 
$\text{PGDli}(\U_0, 8\Gamma, 12\Gamma^{1/2}, \frac{(\sigstarr)^{2}}{108\Gamma^{1/2}}, c, \delta, \frac{r\Gamma^2}{2}, 10\sigstarl)$, with probability $1-\delta$, the output will be $\epsilon$-close to global minima set $\cXstar$ in iterations:
\begin{equation*}
O\left(r\left(\frac{ \Gamma}{\sigstarr}\right)^4\log^{4}\left(\frac{d \Gamma}{\delta\sigstarr}\right)  + \frac{\sigstarl}{\sigstarr}\log \frac{\sigstarr}{\epsilon} \right).
\end{equation*}

~

\noindent\textbf{Part 1:} We will now show part $1$ of the theorem. Recall PGDli (Algorithm \ref{algo:PGDli}) runs PGD (Algorithm \ref{algo:PGD}) first, and then runs gradient descent within $\frac{1}{3}(\sigstarr)^{1/2}$ neighborhood of $\cXstar$. It is easy to verify that $\frac{1}{3}(\sigstarr)^{1/2}$ neighborhood of $\cXstar$ is a subset of $\{\U | \norm{\U}^2 \le \Gamma\}$. Therefore, we only need to show that first phase PGD will not leave the region.
Specifically, we now use induction to prove the following for PGD:
\begin{enumerate}
\item Suppose at iteration $\tau$ we add perturbation, and denote $\tilde{\U}_{\tau}$ to be the iterate before adding perturbation (i.e., $\U_\tau =\tilde{\U}_\tau + \xi_\tau$, and $\tilde{\U}_\tau = \U_{\tau-1} - \eta \grad f(\U_{\tau-1})$). Then, $\norm{\tilde{\U}_{\tau}} \le \frac{1}{2}\Gamma$, and
\item $\norm{\U_t} \le \Gamma$ for all $t\ge 0$. 
\end{enumerate}

By choice of parameters of Algorithm \ref{algo:PGD}, we know $\eta = \frac{c}{ 8\Gamma}$. First, consider gradient descent step without perturbations:
\begin{align*}
\norm{\U_{t+1}} = &\norm{\U_t - \eta \grad f(\U_t)} 
 = \norm{\U_t - \eta (\U_t\U_t\trans - \M^\star)\U_t} \\
 \le & \norm{\U_t - \eta \U_t\U_t\trans \U_t} + \eta\norm{\M^\star\U_t} \\
 \le & \max_i [\sigma_i(\U_t) - \eta \sigma^3_i(\U_t)] + \eta\norm{\M^\star\U_t} 
\end{align*}
For the first term, we know function $f(t) = t - \eta t^3$ is monotonically increasing in $[0, 1/\sqrt{3\eta}]$. On the other hand, by induction assumption, we also know $\norm{\U_t} \le \Gamma^{1/2} \le \sqrt{8\Gamma/(3c)} = 1/\sqrt{3\eta}$. Therefore, the max is taken when $i=1$:
\begin{align}
\norm{\U_{t+1}} \le & \norm{\U_t} - \eta \norm{\U_t}^3 + \eta\norm{\M^\star\U_t}  \nonumber \\
\le& \norm{\U_t} - \eta (\norm{\U_t}^2 - \sigstarl)\norm{\U_t}. \label{eqn:norm-update}
\end{align}
We seperate our discussion into following cases.

\textbf{Case $\norm{\U_t} > \frac{1}{2}\Gamma^{1/2}$:}
In this case $ \norm{\U_t} \ge \max\{\norm{\U_0}, 3(\sigstarl)^{1/2}\}$. Recall $\Gamma^{1/2} = 2\max\{\norm{\U_0}, 3(\sigstarl)^{1/2}\}$. Clearly, $\Gamma \ge 36\sigstarl$, we know:
\begin{align*}
\norm{\U_{t+1}} 
\le& \norm{\U_t} - \eta (\norm{\U_t}^2 - \sigstarl)\norm{\U_t} 
\le \norm{\U_t} - \frac{c}{8\Gamma} (\frac{1}{4}\Gamma - \sigstarl)\frac{1}{2}\Gamma^{1/2}  \\
\le & \norm{\U_t} - \frac{c}{8\Gamma}(\frac{1}{4}\Gamma - \frac{1}{36}\Gamma)  \frac{1}{2}\Gamma^{1/2} 
= \norm{\U_t} - \frac{c}{72}\Gamma^{1/2}.
\end{align*}
This means that in each iteration, the spectral norm would decrease by at least $\frac{c}{72}\Gamma^{1/2}$.

\textbf{Case $\norm{\U_t} \le \frac{1}{2}\Gamma^{1/2}$:}
From~\eqref{eqn:norm-update}, we know that as long as $\norm{\U_t}^2 \ge \norm{\M^\star}$, we will always have $\norm{\U_{t+1}} \le \norm{\U_t} \le \frac{1}{2}\Gamma^{1/2}$. 
For $\norm{\U_t}^2 \le \norm{\M^\star}$, we have:
\begin{align*}
\norm{\U_{t+1}}  \le& \norm{\U_t} - \eta (\norm{\U_t}^2 - \sigstarl)\norm{\U_t} 
= \norm{\U_t} + \frac{c}{8\Gamma} (\sigstarl - \norm{\U_t}^2)\norm{\U_t} \\
\le &\norm{\U_t} + ((\sigstarl)^{1/2} - \norm{\U_t}) \times \frac{c}{8\Gamma} ((\sigstarl)^{1/2} + \norm{\U_t})\norm{\U_t} \\
\le &\norm{\U_t} + ((\sigstarl)^{1/2} - \norm{\U_t}) \times \frac{c\sigstarl}{4\Gamma} \le (\sigstarl)^{1/2}
\end{align*}
Thus, in this case, we always have  $\norm{\U_{t+1}} \le \frac{1}{2}\Gamma^{1/2}$.

In conclusion, if we don't add perturbation in iteration $t$, we have:
\begin{itemize}
\item If $\norm{\U_t} > \frac{1}{2}\Gamma^{1/2}$, then $\norm{\U_{t+1}} \le \norm{\U_t} - \frac{c}{72}\Gamma^{1/2} $.
\item If $\norm{\U_t} \le \frac{1}{2}\Gamma^{1/2}$, then $\norm{\U_{t+1}} \le \frac{1}{2}\Gamma^{1/2} $.
\end{itemize}

Now consider the iterations where we add perturbation. By the choice of radius of perturbation in Algorithm~\ref{algo:PGD} , we increase spectral norm by at most :
\begin{align*}
\norm{\xi_t} \le \fnorm{\xi_t} \le \frac{\sqrt{c}}{\chi^2}\frac{(\sigstarr)^{2}}{108\Gamma^{1/2} \cdot 8\Gamma}\le \frac{1}{2}\Gamma^{1/2}
\end{align*}
The first inequality is because $\chi \ge 1$ and $c\le 1$. 
That is, if before perturbation we have $\norm{\tilde{\U}_t} \le \frac{1}{2}\Gamma^{1/2}$, then $\norm{\U_t} = \norm{\tilde{\U}_t + \xi_t} \le \Gamma^{1/2}$.

On the other hand, according to Algorithm \ref{algo:PGD}, once we add perturbation, we will not add perturbation for next 
$t_{\text{thres}} = \frac{\chi \cdot 24\Gamma}{c^2\sigstarr} \ge \frac{24}{c^2}\ge\frac{48}{c}$ iterations. 
Let $T = \min\{\inf_{i}\{\U_{t+i}| \norm{\U_{t+i}} \le \frac{1}{2}\Gamma^{1/2}\}, t_{\text{thres}}\}$:
\begin{align*}
\norm{\U_{t+T-1}} \le \norm{\U_t} - \frac{c}{72}\Gamma^{1/2} (T-1) 
\le \Gamma^{1/2} (1 - \frac{c(T-1)}{72})
\end{align*}
This gives $T\le \frac{36}{c} < \frac{48}{c} \le t_{\text{thres}}$. Let $\tau>t$ be the next time when we add perturbation ($\tau \ge t + t_{\text{thres}}$), we immediately know $\norm{\U_{T+i}} \le \frac{1}{2}\Gamma^{1/2}$ for $0\le i < \tau-T$ and $\norm{\tilde{\U}_{\tau}} \le \frac{1}{2}\Gamma^{1/2}$.

Finally, $\norm{\U_0} \le \frac{1}{2}\Gamma^{1/2}$ by definition of $\Gamma$, so the initial condition holds. This finishes induction and the proof of the theorem.
\end{proof}

In the next subsections we prove the geometric structures.

\subsection{Smoothness and Hessian Lipschitz}

Before we start proofs of lemmas, we first state some properties about gradient and Hessians. The gradient of the objective function $f(\U)$ is
\begin{equation*}
\grad f(\U) = 2(\U\U\trans - \M^\star)\U.
\end{equation*}

Furthermore, we have the gradient and Hessian satisfy for any $\mZ \in \R^{d \times r}$:
\begin{align}
\la \grad f(\U), \mZ\ra &= 2 \la(\U\U\trans - \M^\star)\U, \mZ\ra, \mbox{ and } \label{eq:gradient} \\
\hess f(\U)(\mZ, \mZ) & = \fnorm{\U\mZ\trans + \mZ\U\trans}^2 + 2\la \U\U\trans - \M^\star, \mZ\mZ\trans \ra. \label{eq:Hessian}
\end{align}

% \textcolor{blue}{Rong: Again I feel we should repeat the lemma.}
\begingroup
\def\thetheorem{\ref{lem:mf_smooth}}
\begin{lemma}
For any $\Gamma \ge \sigstarl$, inside the region $\{\U| \norm{\U}^2 < \Gamma \}$, $f(\U)$ defined in Eq.\eqref{eq:obj} is $8\Gamma$-smooth and $12\Gamma^{1/2}$-Hessian Lipschitz.
\end{lemma}
\addtocounter{theorem}{-1}
\endgroup

\begin{proof}%[Proof of Lemma \ref{lem:mf_smooth}]

Denote $\mathcal{D} = \{\U| \norm{\U}^2 < \Gamma \}$, and recall $\Gamma \ge \sigstarl$. 

~

\noindent \textbf{Smoothness}: For any $\U, \V \in \mathcal{D}$, we have:
\begin{align*}
\fnorm{\grad f(\U) - \grad f(\V)}
=& 2\fnorm{(\U\U\trans - \M^\star)\U - (\V\V\trans - \M^\star)\V} \\
\le& 2\left[\fnorm{\U\U\trans\U - \V\V\trans\V} + \fnorm{\M^\star(\U - \V)}\right]\\
\le& 2\left[3\cdot \Gamma \fnorm{\U-\V} + \sigstarl \fnorm{\U-\V}\right]
\le 8\Gamma\cdot \fnorm{\U-\V}.
\end{align*}

The last line is due to following decomposition and triangle inequality:
\begin{equation*}
\U\U\trans\U - \V\V\trans\V = \U\U\trans(\U - \V) + \U(\U-\V)\trans\V + (\U-\V)\V\trans\V.
\end{equation*}

~

\noindent \textbf{Hessian Lipschitz}: For any $\U, \V \in \mathcal{D}$, and any $\mZ \in \R^{d\times r}$, according to Eq.\eqref{eq:Hessian}, we have:
\begin{align*}
|\hess f(\U)(\mZ, \mZ)  - \hess f(\V)(\mZ, \mZ)|
=& \underbrace{\fnorm{\U\mZ\trans + \mZ\U\trans}^2 - \fnorm{\V\mZ\trans + \mZ\V\trans}^2}_{\mathfrak{A}}+ 
\underbrace{2\la \U\U\trans - \V\V\trans, \mZ\mZ\trans \ra }_{\mathfrak{B}}.
\end{align*}
For term $\mathfrak{A}$, we have:
\begin{align*}
% &\fnorm{\U\mZ\trans + \mZ\U\trans}^2 - \fnorm{\V\mZ\trans + \mZ\V\trans}^2 \\
\mathfrak{A}
= &\la\U\mZ\trans + \mZ\U\trans, (\U-\V)\mZ\trans + \mZ(\U - \V)\trans \ra
+ \la(\U-\V)\mZ\trans + \mZ(\U - \V)\trans ,  \V\mZ\trans + \mZ\V\trans\ra \\
\le & \fnorm{\U\mZ\trans + \mZ\U\trans}\fnorm{(\U-\V)\mZ\trans + \mZ(\U - \V)\trans}
+ \fnorm{(\U-\V)\mZ\trans + \mZ(\U - \V)\trans}\fnorm{ \V\mZ\trans + \mZ\V\trans} \\
\le & 4\norm{\U}\fnorm{\mZ}^2\fnorm{\U - \V} + 4\norm{\V}\fnorm{\mZ}^2\fnorm{\U - \V}
\le 8\Gamma^{1/2} \fnorm{\mZ}^2\fnorm{\U - \V}.
\end{align*}
For term $\mathfrak{B}$, we have:
\begin{align*}
% 2\la \U\U\trans - \V\V\trans, \mZ\mZ\trans 
\mathfrak{B} 
\le 2\fnorm{\U\U\trans - \V\V\trans}\fnorm{\mZ\mZ\trans } \le 4\Gamma^{1/2} \fnorm{\mZ}^2\fnorm{\U - \V}.
\end{align*}
The inequality is due to following decomposition and triangle inequality:
\begin{equation*}
\U\U\trans - \V\V\trans = \U(\U-\V)\trans + (\U-\V)\V\trans.
\end{equation*}
Therefore, in sum we have:
\begin{align*}
\max_{\mZ : \fnorm{\mZ}\le 1} |\hess f(\U)(\mZ, \mZ)  - \hess f(\V)(\mZ, \mZ)| \le& \max_{\mZ : \fnorm{\mZ}\le 1}12 \Gamma^{1/2} \fnorm{\mZ}^2\fnorm{\U - \V} \\
\le & 12 \Gamma^{1/2} \fnorm{\U - \V}.
\end{align*}
\end{proof}

\subsection{Strict-Saddle Property and Local Regularity}

Recall the gradient and Hessian of objective function is calculated as in Eq.\eqref{eq:gradient} and Eq.\eqref{eq:Hessian}. We first prove an elementary inequality regarding to the trace of product of two symmetric PSD matrices. This lemma will be frequently used in the proof.

\begin{lemma}\label{lem:PSD_trace}
For $\A, \B \in \R^{d\times d}$ both symmetric PSD matrices, we have:
\begin{equation*}
\sigma_{\min}(\A) \tr(\B) \le \tr(\A\B)\le \norm{\A} \tr(\B)
\end{equation*}
\end{lemma}
\begin{proof}
Let $\A = \V\D\V\trans$ be the eigendecomposition of $\A$, where $\D$ is diagonal matrix, and $\V$ is orthogonal matrix. Then we have:
\begin{equation*}
\tr(\A\B) = \tr(\D\V\trans\B\V)
=\sum_{i=1}^d \D_{ii} (\V\trans\B\V)_{ii}.
\end{equation*}
Since $\B$ is PSD, we know $\V\trans\B\V$ is also PSD, thus the diagonal entries are non-negative. That is,
$(\V\trans\B\V)_{ii} \ge 0$ for all $i=1, \ldots, d$.
Finally, the lemma follows from the fact that $\sigma_{\min}(\A)\le \D_{ii} \le  \norm{\A}$ and $\tr(\V\trans\B\V) = \tr(\B\V\V\trans) = \tr(\B)$.
\end{proof}

Now, we are ready to prove Lemma \ref{lem:mf_strictsaddle}.

\begingroup
\def\thetheorem{\ref{lem:mf_strictsaddle}}
\begin{lemma}
For $f(\U)$ defined in Eq.\eqref{eq:obj}, all local minima are global minima. The set of global minima is $\cXstar = \{\V^\star \mR | \mR\mR\trans=\mR\trans\mR = \I \}$. Furthermore, $f(\U)$ satisfies:
\begin{enumerate}
\item $(\frac{1}{24}(\sigstarr)^{3/2}, \frac{1}{3}\sigstarr, \frac{1}{3}(\sigstarr)^{1/2})$-strict saddle property, and
\item $(\frac{2}{3}\sigstarr, 10\sigstarl)$-regularity condition in $\frac{1}{3}(\sigstarr)^{1/2}$ neighborhood of $\cXstar$.
\end{enumerate}
\end{lemma} 
\addtocounter{theorem}{-1}
\endgroup

\begin{proof}%[Proof of Lemma \ref{lem:mf_strictsaddle}]

Let us denote the set $\cXstar \defeq \{\V^\star \mR | \mR\mR\trans=\mR\trans\mR = \I \}$, in the end of proof, we will show this set is the set of all local minima (which is also global minima).

Throughout the proof of this lemma, we always focus on the first-order and second-order property for one matrix $\U$. For simplicity of calculation, when it is clear from the context, we denote $\U^\star = \projX(\U)$ and $\Delta = \U - \projX(\U)$. 
By definition of $\cXstar$, we know $\U^\star = \V^\star\mR_\U$ and $\Delta = \U - \V^\star\mR_\U$, where
\begin{equation*}
\mR_\U = \argmin_{\mR:\mR\mR\trans=\mR\trans\mR = \I}\fnorm{\U - \V^\star\mR}^2
\end{equation*}
We first prove following claim, which will used in many places across this proof:
\begin{equation}\label{eq:claim_PSD}
\U\trans\U^\star = \U\trans \V^\star \mR_\U \text{~is a symmetric PSD matrix.}
\end{equation}
This because by expanding the Frobenius norm, and letting the SVD of $\V^\star{}\trans\U$ be $\A\D\B\trans$, we have:
\begin{align*}
&\argmin_{\mR:\mR\mR\trans=\mR\trans\mR = \I}\fnorm{\U - \V^\star\mR}^2
= \argmin_{\mR:\mR\mR\trans=\mR\trans\mR = \I}-\la\U,  \V^\star\mR\ra \\
= &\argmin_{\mR:\mR\mR\trans=\mR\trans\mR = \I}-\tr(\U\trans\V^\star\mR)
= \argmin_{\mR:\mR\mR\trans=\mR\trans\mR = \I}-\tr(\D\A\trans\mR\B)
\end{align*}
Since $\A, \B, \mR$ are all orthonormal matrix, we know $\A\trans\mR\B$ is also orthonormal matrix. Moreover for any orthonormal matrix $\T$, we have:
\begin{equation*}
\tr(\D\T) = \sum_i \D_{ii}\T_{ii} \le \sum_i \D_{ii}
\end{equation*}
The last inequality is because $\D_{ii}$ is singular value thus non-negative, and $\T$ is orthonormal, thus $\T_{ii} \le 1$. This means the maximum of $\tr(\D\T)$ is achieved when $\T = \I$, i.e., the minimum of $-\tr(\D\A\trans\mR\B)$ is achieved when $\mR = \A\B\trans$. Therefore, $\U\trans \V^\star \mR_\U = \B \D\A\trans \A\B\trans = \B\D\B\trans$ is symmetric PSD matrix.

~

\noindent\textbf{Strict Saddle Property}:
In order to show the strict saddle property, we only need to show that for any $\U$ satisfying $\fnorm{\grad f(\U)} \le  \frac{1}{24}(\sigstarr)^{3/2} $
and $\fnorm{\Delta} = \fnorm{\U - \U^\star} \ge \frac{1}{3}(\sigstarr)^{1/2} $, we always have $\sigma_{\min}(\hess f(\U)) \le -\frac{1}{3}\sigstarr$.

Let's consider Hessian $\hess(\U)$ in the direction of $\Delta = \U - \U^\star$. Clearly, we have:
\begin{align*}
&\U\U\trans - \M^\star = \U\U\trans - (\U - \Delta)(\U - \Delta)\trans = (\U\Delta\trans + \Delta\U\trans) -\Delta\Delta\trans 
\end{align*}
and by~\eqref{eq:gradient}:
\begin{align*}
\la \grad f(\U), \Delta\ra =& 2 \la(\U\U\trans - \M^\star)\U, \Delta\ra  =  \la\U\U\trans - \M^\star, \Delta\U\trans + \U\Delta\trans\ra\\
= &\la\U\U\trans - \M^\star, \U\U\trans - \M^\star + \Delta\Delta\trans\ra
\end{align*}
Therefore, by Eq.\eqref{eq:Hessian} and above two equalities, we have:
\begin{align*}
\hess f(\U)(\Delta, \Delta)  = & \fnorm{\U\Delta\trans + \Delta\U\trans}^2 + 2\la \U\U\trans - \M^\star, \Delta\Delta\trans \ra \\
= & \fnorm{\U\U\trans - \M^\star + \Delta\Delta\trans}^2 + 2\la \U\U\trans - \M^\star, \Delta\Delta\trans \ra \\
= & \fnorm{\Delta\Delta\trans}^2 - 3\fnorm{\U\U\trans -\M^\star}^2 + 4\la\U\U\trans - \M^\star, \U\U\trans - \M^\star + \Delta\Delta\trans\ra \\
= & \fnorm{\Delta\Delta\trans}^2 - 3\fnorm{\U\U\trans -\M^\star}^2 + 4\la \grad f(\U), \Delta\ra
\end{align*}

Consider the first two terms, by expanding, we have:
\begin{align*}
& 3\fnorm{\U\U\trans - \M^\star}^2 - \fnorm{\Delta\Delta\trans}^2
= 3\fnorm{(\U^\star\Delta\trans + \Delta\U^\star{}\trans) +\Delta\Delta\trans }^2 - \fnorm{\Delta\Delta\trans}^2 \\
=& 3\cdot \tr\left(2\U^\star{}\trans\U^\star\Delta\trans\Delta + 2(\U^\star{}\trans\Delta)^2 + 4\U^\star{}\trans\Delta\Delta\trans \Delta + (\Delta\trans\Delta)^2\right)
 - \tr ( (\Delta\trans\Delta)^2) \\
=& \tr\left(6\U^\star{}\trans\U^\star\Delta\trans\Delta + 6(\U^\star{}\trans\Delta)^2 + 12\U^\star{}\trans\Delta\Delta\trans \Delta + 2(\Delta\trans\Delta)^2\right)\\
=& \tr((4\sqrt{3}-6)\U^\star{}\trans\U^\star\Delta\trans\Delta + (12-4\sqrt{3})\U^\star{}\trans (\U^\star + \Delta) \Delta\trans\Delta + 2(\sqrt{3}\U^\star{}\trans\Delta + \Delta\trans\Delta)^2) \\
\ge &(4\sqrt{3}-6) \tr(\U^\star{}\trans\U^\star\Delta\trans\Delta) \ge (4\sqrt{3}-6)\sigstarr \fnorm{\Delta}^2
\end{align*}
where the second last inequality is because $\U^\star{}\trans (\U^\star + \Delta) \Delta\trans\Delta = \U^\star{}\trans \U \Delta\trans\Delta$ is the product of two symmetric PSD matrices (thus its trace is non-negative); the last inequality is by Lemma \ref{lem:PSD_trace}. 

Finally, in case we have $\fnorm{\grad f(\U)} \le \frac{1}{24}(\sigstarr)^{3/2} $
and $\fnorm{\Delta} = \fnorm{\U - \U^\star} \ge \frac{1}{3}(\sigstarr)^{1/2} $
\begin{align*}
\sigma_{\min}(\hess f(\U)) \le & \frac{1}{\fnorm{\Delta}^2}\hess f(\U)(\Delta, \Delta)  \le -(4\sqrt{3}-6)\sigstarr + 4\frac{\la \grad f(\U), \Delta\ra}{\fnorm{\Delta}^2}\\
\le & -(4\sqrt{3}-6)\sigstarr + 4\frac{\fnorm{\grad f(\U)}}{\fnorm{\Delta}}
\le -(4\sqrt{3}-6.5)\sigstarr \le -\frac{1}{3}\sigstarr
\end{align*}

~

\noindent\textbf{Local Regularity}:
In $\frac{1}{3}(\sigstarr)^{1/2}$ neigborhood of $\cXstar$, by definition, we know, 
\begin{align*}
\fnorm{\Delta}^2 = \fnorm{\U - \U^\star}^2 \le \frac{1}{9}\sigstarr.
\end{align*}

Clearly, by Weyl's inequality, we have $\norm{\U} \le \norm{\U^\star} + \norm{\Delta} \le \frac{4}{3}(\sigstarl)^{1/2}$, and $\sigma_r(\U) \ge \sigma_r(\U^\star) - \norm{\Delta} \ge \frac{2}{3}(\sigstarr)^{1/2}$. Moreover, since $\U^\star{}\trans\U$ is symmetric matrix, we have:
\begin{align*}
\sigma_r(\U^\star{}\trans\U) = & \frac{1}{2}\left(\sigma_r(\U\trans\U^\star + \U^\star{}\trans\U)\right) \\
\ge & \frac{1}{2}\left(\sigma_r(\U\trans\U + \U^\star{}\trans\U^\star) - \norm{(\U-\U^\star)\trans(\U-\U^\star)}\right) \\
\ge & \frac{1}{2}\left(\sigma_r(\U\trans\U) + \sigma_r(\U^\star{}\trans\U^\star) - \fnorm{\Delta}^2\right) \\
\ge & \frac{1}{2}(1 + \frac{4}{9} - \frac{1}{9}) \sigstarr = \frac{2}{3} \sigstarr.
\end{align*}

At a highlevel, we will prove $(\alpha, \beta)$-regularity property~\eqref{eq:regularity} by proving that:
\begin{enumerate}
\item $\la \grad f(\x), \x - \projX(\x) \ra \ge \alpha\norm{\x- \projX(\x)}^2$, and
\item $\la \grad f(\x), \x - \projX(\x) \ra \ge \frac{1}{\beta} \norm{\grad f(\x)}^2$.
\end{enumerate} 
%Combining them, we get standard regularity condition.

According to~\eqref{eq:gradient}, we know:
\begin{align}
\la \grad f(\U), \U - \projX(\U) \ra =& 2 \la(\U\U\trans - \M^\star)\U, \Delta\ra 
= 2 \la\U\Delta\trans + \Delta \U^\star{}\trans, \Delta\U\trans\ra \nn\\
=& 2 (\tr(\U\Delta\trans\U\Delta\trans) + \tr(\Delta \U^\star{}\trans\U\Delta\trans)) \nn\\
=& 2 (\fnorm{\Delta\trans\U}^2 + \tr(\U^\star{}\trans\U\Delta\trans\Delta )). \label{eq:inner_grad_distance}
\end{align}
The last equality is because $\Delta\trans \U$ is symmetric matrix. Since $\U^\star{}\trans\U$ is symmetric PSD matrix, and recall $\sigma_r(\U^\star{}\trans\U) \ge \frac{2}{3} \sigstarr$, by Lemma \ref{lem:PSD_trace} we have:
\begin{equation}\label{eq:reg_part1}
\la \grad f(\U), \U - \projX(\U) \ra \ge \sigma_r(\U^\star{}\trans\U) \tr(\Delta\trans\Delta)
\ge\frac{2}{3}\sigstarr \fnorm{\Delta}^2.
\end{equation}
On the other hand, we also have:
\begin{align*}
\fnorm{\grad f(\U)}^2 =& 4\la(\U\U\trans - \M^\star)\U, (\U\U\trans - \M^\star)\U\ra \\
=& 4\la(\U\Delta\trans + \Delta\U^\star{}\trans)\U, (\U\Delta\trans + \Delta\U^\star{}\trans)\U\ra \\
=& 4\left(\underbrace{\tr[(\Delta\trans\U\U\trans\Delta) \U\trans\U] + 2\tr[\Delta\trans\U \U\trans\U^\star \Delta\trans\U]}_{\mathfrak{A}} 
+ \underbrace{\tr(\U^\star{}\trans\U \U\trans \U^\star \Delta\trans\Delta)}_{\mathfrak{B}}  \right).
\end{align*}
For term $\mathfrak{A}$, by Lemma \ref{lem:PSD_trace}, and $\Delta\trans\U$ being a symmetric matrix, we have:
\begin{align*}
\mathfrak{A} \le \norm{\U\trans\U}\fnorm{\Delta\trans\U}^2 + 2\norm{\U\trans\U^\star}\fnorm{\Delta\trans\U}^2
\le (\frac{16}{9} + \frac{8}{3}) \sigstarl \fnorm{\Delta\trans\U}^2 \le 5\sigstarl \fnorm{\Delta\trans\U}^2
\end{align*}
For term $\mathfrak{B}$, by Eq.\eqref{eq:claim_PSD} we can denote $\C = \U^\star{}\trans\U = \U\trans \U^\star$ which is symmetric PSD matrix, by Lemma \ref{lem:PSD_trace}, we have:
\begin{align*}
\mathfrak{B} = & \tr(\C^2 \Delta\trans\Delta) = \tr(\C (\C^{1/2}\Delta\trans\Delta\C^{1/2})) \\
\le & \norm{\C}\tr(\C^{1/2}\Delta\trans\Delta\C^{1/2}) = \norm{\C}\tr(\C\Delta\trans\Delta) 
\le \frac{4}{3}\sigstarl \tr(\U^\star{}\trans\U\Delta\trans\Delta ).
\end{align*}
Combining with~\eqref{eq:inner_grad_distance} we have:
\begin{equation}\label{eq:reg_part2}
\fnorm{\grad f(\U)}^2 \le \sigstarl (20\fnorm{\Delta\trans\U}^2 + \frac{16}{3}\tr(\U^\star{}\trans\U\Delta\trans\Delta ))
\le 10 \sigstarl \la \grad f(\U), \U - \projX(\U) \ra.
\end{equation}
Combining~\eqref{eq:reg_part1} and~\eqref{eq:reg_part2}, we have:
\begin{equation*}
\la \grad f(\U), \U - \projX(\U) \ra  \ge \frac{1}{3}\sigstarr \fnorm{\U - \projX(\U)}^2 + \frac{1}{20 \sigstarl}\fnorm{\grad f(\U)}^2.
\end{equation*}
\end{proof}

\end{document}